\documentclass{article}

\usepackage{microtype}
\usepackage{graphicx}
\usepackage{subcaption}
\usepackage{booktabs} % for professional tables
\usepackage{multirow}
\usepackage{xcolor} 
\usepackage{placeins}
\usepackage{float}
\usepackage{colortbl}

\usepackage{hyperref}
\usepackage{listings}

\usepackage{listings}
\usepackage{xcolor}
\usepackage{inconsolata} % 【核心】更美观的等宽字体
\definecolor{graybg}{gray}{0.95}
\definecolor{ebgreen}{rgb}{0.0, 0.6, 0.0}   % 关键字绿色 (def, return)
\definecolor{ebblue}{rgb}{0.0, 0.0, 1.0}    % 函数名蓝色
\definecolor{ebgray}{rgb}{0.5, 0.5, 0.5}    % 注释灰色

\lstdefinestyle{ebstyle}{
    language=Python,
    basicstyle=\ttfamily\small,        % 字体大小
    keywordstyle=\color{ebgreen}\bfseries, % 关键字样式
    deletekeywords={max, sum, log, exp, min, print},
    commentstyle=\color{ebgray},       % 注释样式
    stringstyle=\color{ebgray},
    numbers=left,                      % 行号在左
    numberstyle=\tiny\color{ebgray},   % 行号样式
    stepnumber=1,
    numbersep=5pt,
    frame=none,                        % 【重点】去掉边框
    showspaces=false,
    showstringspaces=false,
    tabsize=4,
    xleftmargin=10pt,                  % 左边距，防止行号贴边
    emph={step, swd_step},             % 强调函数名
    emphstyle=\color{ebblue}\bfseries  % 函数名变为蓝色
}

\usepackage[preprint]{icml2026}

\usepackage{amsmath}
\usepackage{amssymb}
\usepackage{mathtools}
\usepackage{amsthm}

\usepackage[capitalize,noabbrev]{cleveref}

\theoremstyle{plain}
\newtheorem{theorem}{Theorem}[section]

\newtheorem{lemma}[theorem]{Lemma}

\theoremstyle{definition}

\theoremstyle{remark}

\usepackage[textsize=tiny]{todonotes}

\begin{document}

\twocolumn[
  \icmltitle{Stability-Weighted Decoding for Diffusion Language Models}

  % 移除 icmlsetsymbol{equal}{*} 因为不再需要标注共同一作

  \begin{icmlauthorlist}
    % 第一作者
    \icmlauthor{Yue Wu}{yyy}
    % 第二作者（通讯作者）
    \icmlauthor{Jian Huang}{yyy} 
  \end{icmlauthorlist}

  % 单位定义
  \icmlaffiliation{yyy}{Department of Data Science and Artificial Intelligence, The Hong Kong Polytechnic University, Hong Kong, China}

  % 通讯作者邮箱设置（Jian Huang 作为主要的 Correspondence）
  \icmlcorrespondingauthor{Jian Huang}{j.huang@polyu.edu.hk}
%   \icmlcorrespondingauthor{Yue Wu}{yue0301.wu@connect.polyu.hk}

  \icmlkeywords{Diffusion Large Language Models, Parallel Decoding, Stability-Weighted Decoding, dLLM}

  \vskip 0.3in
]



\printAffiliationsAndNotice{}  % no special notice (required even if empty)

\begin{abstract}
Diffusion large language models (dLLMs) enable parallel text generation by iteratively denoising a fully masked sequence, unmasking a subset of masked tokens at each step.
Existing decoding strategies rely on static confidence metrics computed at a single denoising step, ignoring temporal history and often leading to premature unmasking of unstable tokens.
In this work, we theoretically establish that a token's temporal instability, quantified by the KL divergence between consecutive prediction distributions, provides a strict lower bound on its mutual information with the remaining masked context, indicating that temporally unstable tokens are inherently unsafe to unmask.
Based on this insight, we propose Stability-Weighted Decoding (SWD), a training-free, plug-and-play strategy that incorporates temporal stability into token scoring and acts as a universal modulator for arbitrary score-based decoding policies.
Experiments on code generation and mathematical reasoning benchmarks demonstrate that SWD consistently improves generation accuracy across representative scoring metrics and selection policies, and exhibits exceptional robustness, maintaining a significant performance lead over standard baselines across varying acceleration ratios.

\end{abstract}

\section{Introduction}
\label{sec:intro}

\begin{figure*}[t]
    \centering
    % --- (a) 结果展示 (Case Study) ---
    % 宽度设为 0.32\textwidth，留出一点空隙
    \begin{subfigure}[b]{0.32\textwidth} % 保持宽度一致
        \centering
        % 【修改点】：使用 \raisebox{高度}{图片}
        % 0.15\height 表示向上移动图片自身高度的 15%
        % 您也可以直接写具体数值，如 0.5cm, 10pt 等
        \raisebox{0.15\height}{
            \includegraphics[width=\textwidth]{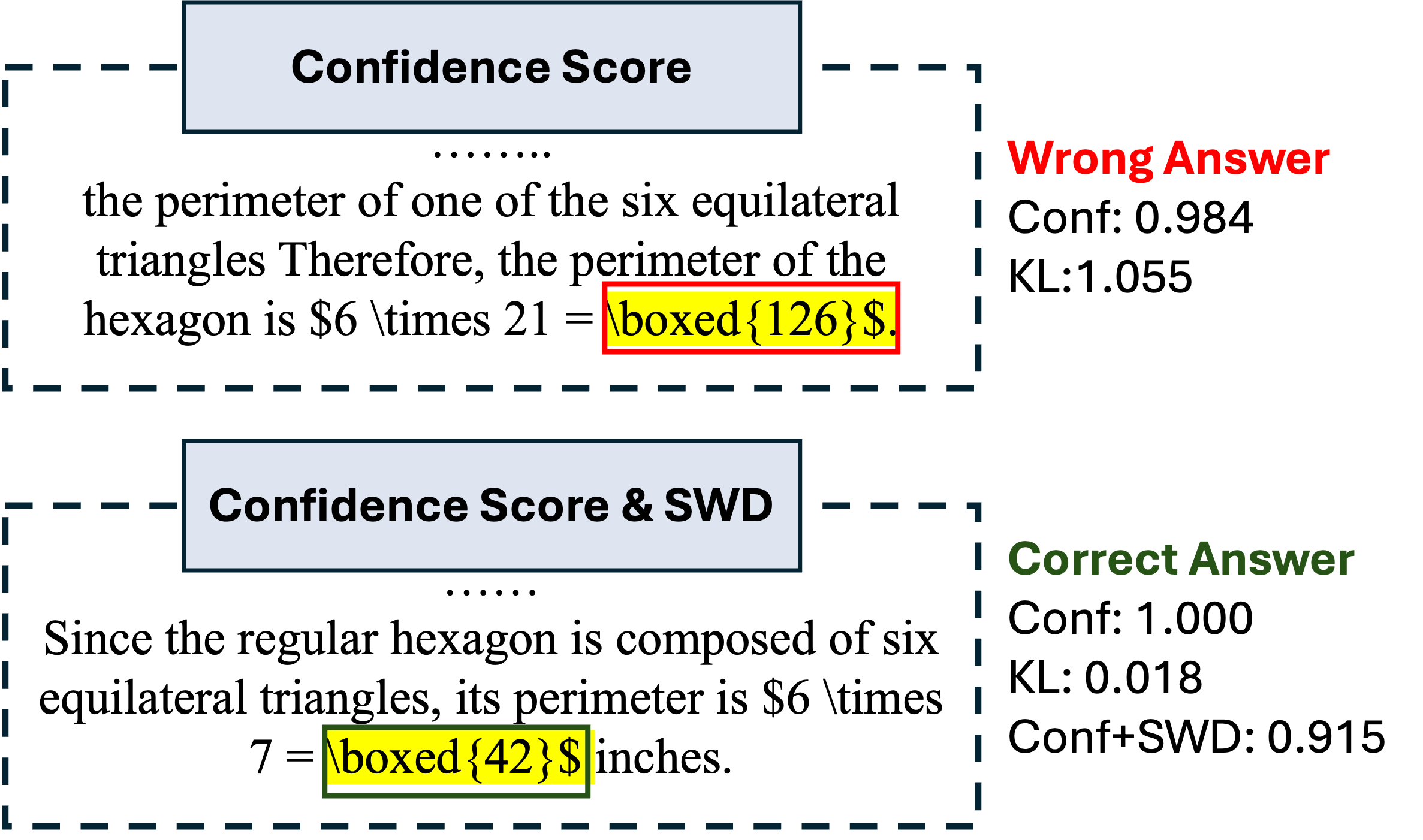}
        }
        
        % 如果图片上移后，图片和 Caption 中间空隙太大，
        % 可以取消下面这行的注释来把 Caption 拉回来一点
        % \vspace{-5pt} 
        
        \caption{Case Study: Hexagon Problem}
        \label{fig:case_study}
    \end{subfigure}
    % \hfill % 在子图之间自动填充空白
    % --- (b) Baseline 机制分析 ---
    \begin{subfigure}[b]{0.32\textwidth}
        \centering
        \includegraphics[width=\textwidth]{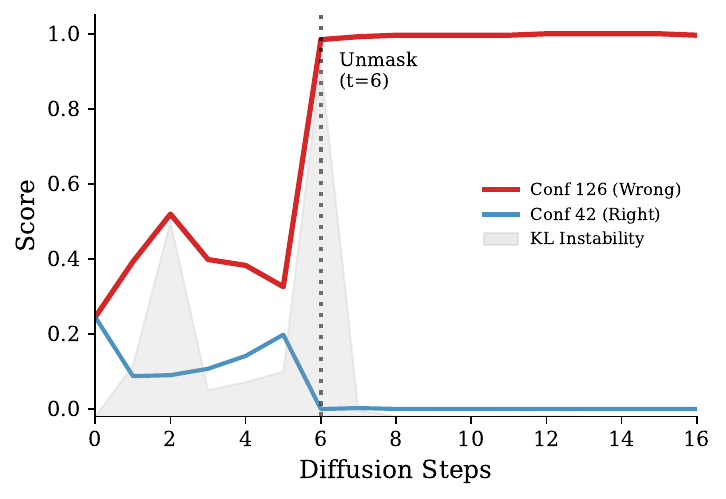}
        \caption{Baseline: Premature Commitment}
        \label{fig:anatomy_base}
    \end{subfigure}
    \hfill % 在子图之间自动填充空白
    % --- (c) SWD 机制分析 ---
    \begin{subfigure}[b]{0.32\textwidth}
        \centering
        \includegraphics[width=\textwidth]{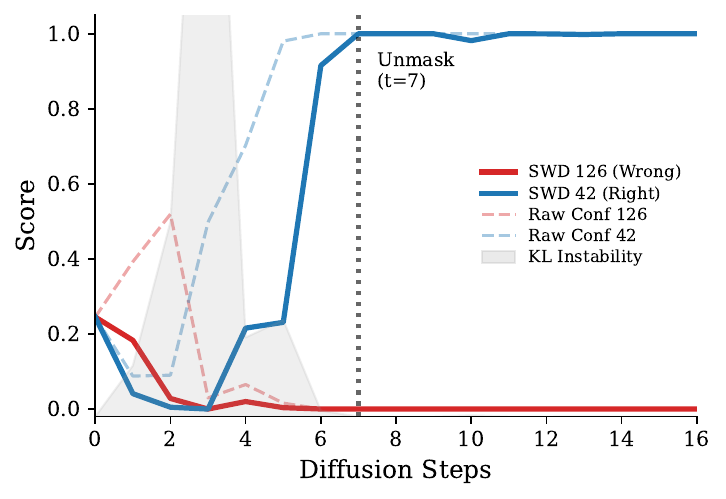}
        \caption{SWD: Stability-Aware Selection}
        \label{fig:anatomy_swd}
    \end{subfigure}
    
    % 调整 Caption 与图片的距离
    \vspace{-5pt}
    
    \caption{We visualize the decoding trajectory for a hexagon geometry problem in Math500. 
    \textbf{(a)} The Baseline (Confidence-only) generates an incorrect answer "126", while SWD correctly outputs "42". 
    \textbf{(b)} Although the incorrect token '126' exhibits high instability (large KL divergence, grey area), the model prematurely unmasks it at $t=6$ due to transient high confidence. 
    \textbf{(c)} SWD applies a stability penalty. The high instability of '126' suppresses its score (Red solid line), preventing the error. The correct token '42' is selected only after it stabilizes later in the diffusion process.}
    \label{fig:qualitative_analysis}
\end{figure*}

% Discrete diffusion show potential in various generative task, especially in text generation. Among them, masked diffusion model\cite{simple_effective} show great performance, and scale it up to large language model\cite{llada, dream}, challenging the dominance of autoregressive model. Different from autoregressive model which generate token one by one, diffusion language model(dLLM) use bidirection attention, non-autoregressive, iterative denoising process, could be viewed as any-order autoregressive\cite{ou2024absorbingdiscretediffusionsecretlyradd, trainkim}. This mechanism allows bidirectional context utilization enables parallel decoding — the ability to predict and unmask multiple tokens simultaneously\cite{ebsampler, fastdllm}. While promising, the efficiency and quality of parallel decoding critically depend on the selection policy: deciding which subset of tokens is sufficiently reliable to be unmasked at each step.
Discrete diffusion models \cite{d3pm} have emerged as a compelling paradigm for generative tasks, with Masked Diffusion Models (MDMs) \cite{simple_effective} demonstrating exceptional performance in text generation. 
By scaling this specific architecture, recent works like LLaDA \cite{llada} and Dream \cite{dream} have established the class of Diffusion Large Language Models (dLLMs)\cite{llada2.0, llada1.5,diffucoder}, challenging the dominance of autoregressive (AR) methods. 
Unlike AR models constrained to left-to-right processing, dLLMs leverage bidirectional attention and an iterative denoising process, and conceptually characterized as an any-order autoregressive process \cite{ou2024absorbingdiscretediffusionsecretlyradd, trainkim}, possess higher potential than their fixed-order counterparts. 
This unique mechanism enables parallel decoding capability to predict and unmask multiple tokens simultaneously \cite{ebsampler, fastdllm}. 
However, realizing this potential requires a critical decision at each step: accurately identifying which subset of tokens is sufficiently reliable to unmask.

Current decoding strategies predominantly rely on static confidence metrics (e.g., confidence\cite{llada}, margin\cite{trainkim}, entropy\cite{koh2025conditionalmaskdiscretediffusionentropy}) derived from a single timestep. 
We argue that these "snapshot-based" methods exhibit a fundamental blindness: they treat the dynamic denoising process as independent events, neglecting the token's historical trajectory. As illustrated in Figure \ref{fig:qualitative_analysis}, tokens often exhibit temporal instability — transient high confidence with significant fluctuations. 
Relying solely on static metrics may lead to premature commitment, where unstable tokens are unmasked too early, locking in errors that propagate globally.

In this work, we formalize the role of temporal dynamics through a rigorous information-theoretic lens. We theoretically establish that a token's temporal instability—quantified by the Kullback-Leibler (KL) divergence between consecutive steps—constitutes a strict lower bound on its mutual information with the remaining masked context. This implies that if a token's distribution shifts drastically given a refined context, it remains entangled with the unknown future and is inherently unsafe to unmask.

Guided by this insight, we introduce Stability-Weighted Decoding (SWD), a training-free, plug-and-play decoding strategy. We fundamentally reformulate token selection as a risk-constrained optimization problem: maximizing confidence utility subject to a strict stability constraint. 
This formulation mathematically yields a closed-form modulation mechanism, where the base score is weighted by an exponential decay term proportional to its instability ($\exp(-\lambda \cdot D_{\text{KL}})$). By suppressing the scores of unstable candidates, SWD effectively prevents the model from committing to temporally unconverged tokens, ensuring that unmasking decisions are grounded in both confidence and stability.

We validate SWD across two state-of-the-art dLLMs (LLaDA-8B-Instruct\cite{llada} and Dream-v0-Instruct-7B\cite{dream}) on diverse benchmarks, including code generation (HumanEval\cite{humaneval}, MBPP\cite{mbpp}) and mathematical reasoning (GSM8K\cite{gsm8k}, MATH500\cite{math500}). Our contributions are threefold:

\begin{itemize}
    \item We propose \textbf{Stability-Weighted Decoding (SWD)}, a training-free, plug-and-play strategy that acts as a universal modulator for arbitrary score-based decoding policies. By explicitly suppressing the scores of unstable tokens, SWD effectively prevents the model from committing to statistically plausible but temporally unconverged candidates.
    
    \item We theoretically prove that temporal instability constitutes a strict lower bound on the mutual information between a token and the unrevealed context. This provides a rigorous mathematical justification for using stability constraints to prevent premature unmasking.
    
    \item We provide comprehensive empirical validation across diverse models, datasets, scoring metrics, and selection strategies. Results demonstrate that SWD consistently enhances performance across these varied settings, highlighting its exceptional robustness and versatility.
\end{itemize}

% \begin{figure*}[h]
%     \centering
%     % --- 上面的图 ---
%     \begin{subfigure}{\linewidth}
%         \centering
%         % 请替换文件名
%         \includegraphics[width=\linewidth]{figs/pipeline1.png} 
%         % \caption{Pipeline of dLLM with confidence score and Top-1 selection strategy.} % 上图的子标题
%         \label{fig:pipeline_top}
%     \end{subfigure}
    
%     % \par\bigskip % 添加一些垂直间距
    
%     % --- 下面的图 ---
%     \begin{subfigure}{\linewidth}
%         \centering
%         % 请替换文件名
%         \includegraphics[width=\linewidth]{figs/pipeline2.png}
%         % \caption{Pipeline of dLLM with confidence SWD enhanced score and Top-1 selection strategy} % 下图的子标题
%         \label{fig:pipeline_bottom}
%     \end{subfigure}
    
%     % --- 总标题 ---
%     \caption{Overview of SWD with confidence score and Top-1 selection strategy.}
%     \label{fig:pipeline}
% \end{figure*}

\begin{figure*}[t]
    \centering
    % --- 上面的图 (Baseline) ---
    \begin{subfigure}{\linewidth}
        \centering
        \includegraphics[width=\linewidth]{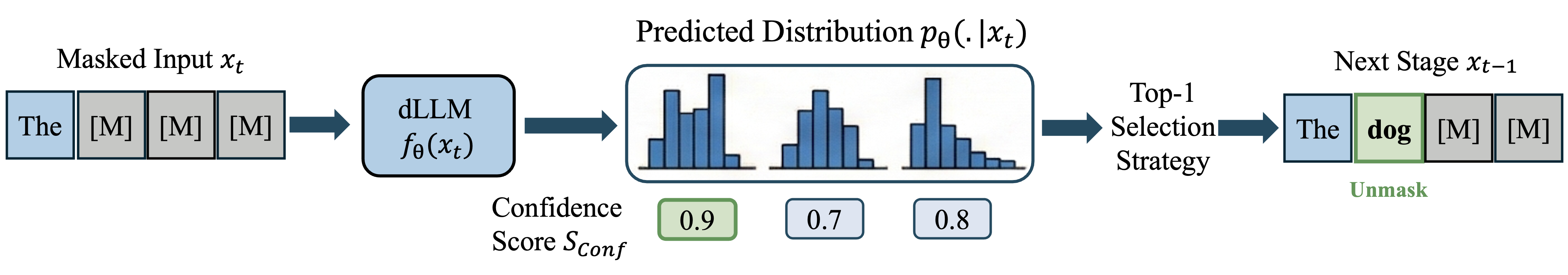} 
        \label{fig:pipeline_top}
    \end{subfigure}
    
    % \vspace{3pt} % 稍微加一点间距，把上下两张图隔开
    
    % --- 下面的图 (SWD) ---
    \begin{subfigure}{\linewidth}
        \centering
        \includegraphics[width=\linewidth]{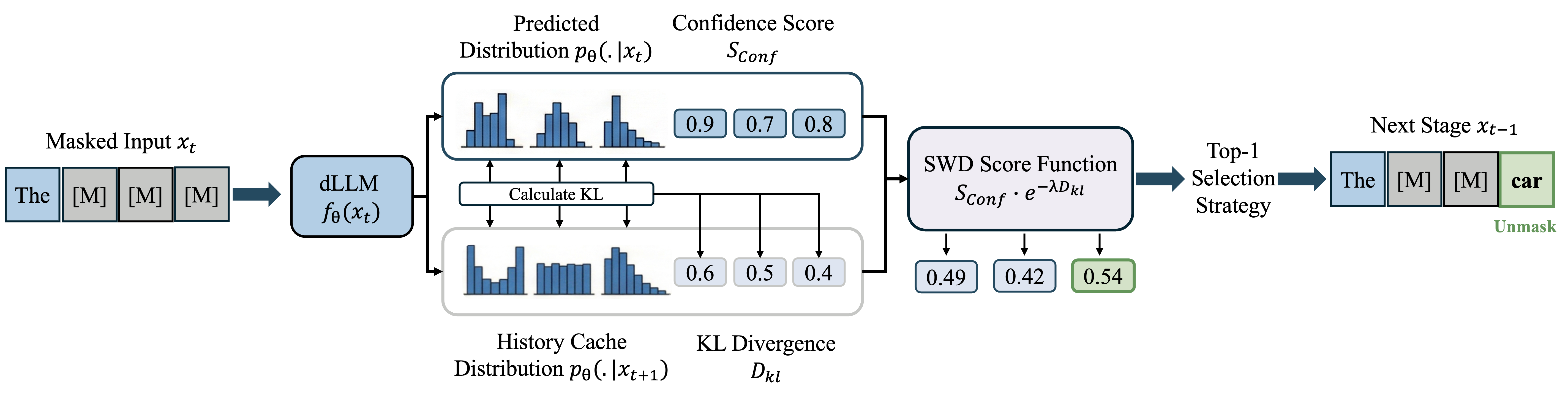}
        \label{fig:pipeline_bottom}
    \end{subfigure}
    
    % --- 总标题 (Unified Caption) ---
    \caption{\textbf{Overview of decoding pipelines using Confidence Score and Top-1 Selection.} 
    \textbf{(Top) Standard dLLM:} The model relies solely on the static confidence score at the current step. In this case, it prematurely commits to the high-confidence token "dog" ($0.9$) despite its instability.
    \textbf{(Bottom) Stability-Weighted Decoding (SWD):} By incorporating the historical distribution, SWD identifies the instability of "dog" (high KL divergence) and penalizes its score ($0.9 \to 0.49$). Consequently, the selection shifts to the more stable candidate "car" ($0.8 \to 0.54$), effectively correcting the error.}
    \label{fig:pipeline}
\end{figure*}

\section{Related Work}
\subsection{Diffusion Language Models}
The application of diffusion models to discrete data was pioneered by D3PM \cite{d3pm}, laying the foundation for discrete diffusion language models (dLLMs). Subsequent works, including CTMC \cite{ctmc}, SEDD \cite{sedd}, and Block Diffusion \cite{blockdiffusion}, further formalized this framework and introduced various extensions. Among these approaches, masked diffusion models, which adopt absorbing states, have demonstrated particularly strong empirical performance \cite{simple_effective}. This paradigm has recently been scaled to large instruction-tuned models, such as the LLaDA and Dream families \cite{llada, dream}, challenging the long-standing dominance of autoregressive language models. 
More recent efforts have continued to enhance dLLMs in terms of capability and scalability \cite{diffucoder, llada1.5, llada2.0, seed}, and extending them to multimodal settings \cite{lladav, mmada, dimple}. In addition, several studies have provided theoretical analyses that investigate the efficiency and limitations of dLLMs from different perspectives \cite{limitationofdiff, timeagnostic, trainkim}.

\subsection{Decoding Strategies for Diffusion Language Models}
Diffusion language models generate text by iteratively denoising a fully masked sequence, where decoding quality and efficiency are governed by how masked tokens are selected for unmasking at each iteration. Early approaches adopt random or uniform unmasking schedules \cite{llada}, while later work introduces score-based criteria to prioritize positions according to model uncertainty signals, such as confidence, entropy\cite{koh2025conditionalmaskdiscretediffusionentropy}, or margin \cite{trainkim}, as well as more sophisticated scoring functions \cite{pcsampler}. 
To accelerate generation, several methods unmask multiple tokens per step using threshold-based rules \cite{fastdllm, klass} or entropy-bound criteria \cite{ebsampler}. 
Recently, some studies exploit temporal or historical information during decoding, showing that tokens whose predictions remain stable across iterations are more reliable candidates for unmasking, thereby improving decoding robustness without altering the underlying diffusion process \cite{klass, creditdecoding, timeisfeature, beyondconf}. 
Another line of work explicitly introduces remasking strategies, which allow previously unmasked tokens to be masked again and revised in later iterations to correct early mistakes or improve global coherence \cite{pathplanning, wino}. Orthogonal to decoding policy design, several works focus on engineering optimizations tailored to diffusion-style generation, such as KV-cache mechanisms that reduce redundant computation and accelerate inference \cite{fastdllmv2, dkvcache, d2cache}.

\section{Preliminaries}
\label{sec:preliminaries}

In this section, we formalize the notation, review the framework of Discrete Diffusion Language Models, and discuss existing parallel decoding methods.

\subsection{Discrete Diffusion Language Models}
\textbf{Notation}
Let $\mathcal{V}$ denote a vocabulary set of size $K = |\mathcal{V}|$. We denote $\mathbf{x}_t = (x_t^1, \dots, x_t^L)$ as a sequence of length $L$ at time step $t$, where each token $x_t^i \in \mathcal{V} \cup \{[\texttt{MASK}]\}$. 
The diffusion process operates over discrete time steps $t \in \{0, \dots, T\}$, where $t=0$ represents the clean data $\mathbf{x}_0$ and $t=T$ represents the fully masked state $\mathbf{x}_T = [\texttt{MASK}]^L$. 
We let $M_t \subseteq \{1, \dots, L\}$ denote the set of masked indices at step $t$, $\alpha_t \in [0, 1]$ represents the mask schedule (the proportion of masked tokens), satisfying $0 = \alpha_0 < \dots < \alpha_T = 1$.

\textbf{Reverse Process.}

The decoding (reverse) process iteratively denoises the sequence from the fully masked state $\mathbf{x}_T$ to progressively less corrupted states until reaching the clean sequence $\mathbf{x}_0$. 
The core component of a dLLM is the denoising network $f_\theta$, which maps a masked sequence $\mathbf{x}_t$ to a logit tensor $f_\theta(\mathbf{x}_t) \in \mathbb{R}^{L \times K}$. 
The network predicts the original clean token $x_0^i$ directly, leading the categorical distribution:
\begin{equation}
    p_\theta^i(\cdot \mid \mathbf{x}_t) = \operatorname{Cat}\Big( \operatorname{Softmax}\big( [f_\theta(\mathbf{x}_t)]^i \big) \Big).
    \label{eq:denoising_net}
\end{equation}

Following \cite{simple_effective}, the posterior  distribution for the reverse step is derived as:

\begin{equation}
    p_\theta(x_{t-1}^i \mid \mathbf{x}_t) = 
    \begin{cases} 
    % 使用 aligned 包裹第一行的内容
    \begin{aligned}
        & \operatorname{Cat}\Big( \tfrac{\alpha_{t-1}}{\alpha_t} \mathbf{e}_{\text{M}} + \tfrac{\alpha_t - \alpha_{t-1}}{\alpha_t} p_\theta^i(\cdot \mid \mathbf{x}_t) \Big) \\
        & \hspace{12em} \text{if } i \in M_t 
    \end{aligned} \\
    % 第二行保持原样
    \delta(x_{t-1}^i = x_t^i) \qquad \ \ \hspace{4em} \text{if } i \notin M_t
    \end{cases}
    \label{eq:reverse_transition}
\end{equation}

where $\mathbf{e}_{\textbf{M}}$ denotes the one-hot vector of the masked token. Eq. (\ref{eq:reverse_transition}) implies for each masked position $i \in M_t$, the token is either unmasked, sampled from the predicted clean distribution, with probability $(\alpha_t - \alpha_{t-1})/\alpha_t$, or remains masked with the complementary probability.

\subsection{Generalized Parallel Decoding Framework}
\label{sec:generalized_decoding}

To accelerate inference, modern dLLMs adopt a parallel decoding paradigm. This process selects a subset of masked tokens $S_t \subseteq M_t$ to unmask simultaneously, thereby reducing the number of iterations and improving efficiency. In multi-token prediction process, predictions at masked positions are assumed conditionally independent given $x_t$, for any subset $S_t \subseteq M_t$ to unmask, we assume the joint distribution be approximated by the product of marginals:

\begin{equation}
    q_\theta(\mathbf{x}^{S_t} \mid \mathbf{x}_t) \approx \prod_{i \in S_t} p_{\theta}^{i}(x^{i} \mid \mathbf{x}_{t}).
    \label{eq:multiproduct}
\end{equation}
% where $\mathbf{x}_t^S = \{  x_t^i\} _ {i\in S_t}$. Prior theoretical work\cite{fastdllm} shows the product of marginals is close to the true joint distribution when confidence score(which is defined below) is close to 1, shows that parallel decoding is effective to sequential greedy decoding.

Prior theoretical work \cite{fastdllm} demonstrates that this product of marginals closely approximates the true joint distribution when the model's confidence score (which is defined below) is close to 1, validating the parallel decoding effectiveness is equivalent to sequential greedy decoding.

Formally, parallel decoding consists of two stages: Scoring and Selection.

\textbf{Scoring Functions.} 
% Existing methods evaluate the "confidence" of each masked position $i \in M_t$ using a heuristic score $\mathcal{S}(i)$ derived from the predicted distribution $p_\theta(\cdot \mid \mathbf{x}_t)_i$. Common choices include:

Methods evaluate the reliability of each masked position $i \in M_t$ using score $\mathcal{S}(i)$ derived from $p_\theta(\cdot \mid \mathbf{x}_t)_i$. Common metrics include:

\begin{itemize}
    \item \textbf{Confidence:} $\mathcal{S}_{\text{conf}}(i) = \max_{v \in \mathcal{V}} p_\theta(x_0^i=v \mid \mathbf{x}_t)$. This measures the probability of the mode; a higher score indicates greater model certainty.
    \item \textbf{Entropy:} $\mathcal{S}_{\text{ent}}(i) = -H(p_\theta^i(\cdot \mid \mathbf{x}_t))$. Low entropy implies a peaked distribution and high certainty.
    \item \textbf{Margin:} $\mathcal{S}_{\text{margin}}(i) = p_\theta^i(v_{(1)} \mid \mathbf{x}_t) - p_\theta^i(v_{(2)} \mid \mathbf{x}_t)$, where $v_{(1)}$ and $v_{(2)}$ are the top-1 and top-2 predicted tokens. This reflects the discriminative power of the model; a higher score means more confident, less curious
\end{itemize}

\textbf{Selection Strategies.} 
Given the scores, a policy determines the subset of tokens $S_t \subset M_t$ to unmask.

\begin{itemize}
    \item \textbf{Top-$k$.} The simplest approach is to unmask a fixed number of tokens $k$ with the highest scores at each step.

\item \textbf{Threshold-based.} Policies like Fast-dLLM \cite{fastdllm} select all positions where the score exceeds the threshold $\tau$, i.e., $S_t = \{i \in M_t \mid \mathcal{S}(i) > \tau\}$.

% \textbf{Adaptive policies} (e.g., EB-Sampler \citep{ebsampler}) rank positions by their scores and select the top candidates until a cumulative uncertainty budget is met.

% $\sum_{l \in U} H(p^{\theta}(x^{l} \mid x^{\bar{M}})) - \max_{l \in U} H(p^{\theta}(x^{l} \mid x^{\bar{M}})) \leq \gamma.$
\item \textbf{EB-Sampler.} Method EB-Sampler \cite{ebsampler} dynamically determine the batch size by rank positions by their scores and select the top candidates until a cumulative uncertainty budget $\gamma$ is met.
\begin{equation}
    \sum_{i \in S_t} H(p^{i}_\theta) - \max_{j \in S_t} H(p^{j}_\theta) \leq \gamma,
    \label{eq:eb_bound}
\end{equation}

\end{itemize}

\begin{algorithm}[tb]
   \caption{Stability-Weighted Decoding (SWD)}
   \label{alg:swd}
\begin{algorithmic}
   \STATE {\bfseries Input:} Masked sequence $\mathbf{x}_T$, denoiser $f_\theta$, penalty $\lambda$
   \STATE {\bfseries Output:} Generated sequence $\mathbf{x}_0$
   \STATE \textbf{Initialize:} $\mathbf{p}_{T+1} \leftarrow \text{Uniform}(\mathcal{V})$, $M_T \leftarrow \{1, \dots, L\}$
   
   \FOR{$t = T$ {\bfseries down to} $1$}
       \STATE \textbf{// 1. Predict \& Score}
       \STATE Compute probabilities: $\mathbf{p}_t \leftarrow \operatorname{softmax}(f_\theta(\mathbf{x}_t))$
       \STATE Compute base scores: $\mathcal{S} \leftarrow \Phi(\mathbf{p}_t)$ \textit{(e.g., Confidence)}
       
       \STATE \textbf{// 2. Stability Modulation (SWD)}
       \STATE Compute instability: $\mathcal{D}_{\text{KL}} \leftarrow \sum \mathbf{p}_{t+1} \log(\mathbf{p}_{t+1} / \mathbf{p}_t)$
       \STATE Update scores: $\mathcal{S} \leftarrow \mathcal{S} \cdot \exp(-\lambda \cdot \mathcal{D}_{\text{KL}})$
       
       \STATE \textbf{// 3. Select \& Update (Example: Top-1)}
       \STATE Identify best candidate: $i^* \leftarrow \arg\max_{i \in M_t} \mathcal{S}^{(i)}$
       \STATE Unmask: $x_{t-1}^{(i^*)} \leftarrow \arg\max \mathbf{p}_t^{(i^*)}$
       \STATE Update Mask: $M_{t-1} \leftarrow M_t \setminus \{i^*\}$
       \STATE Cache history: $\mathbf{p}_{t+1} \leftarrow \mathbf{p}_{t}$
   \ENDFOR
   \STATE \textbf{return} $\mathbf{x}_0$
\end{algorithmic}
\end{algorithm}

\section{Methodology}
\label{sec:method}

We propose Stability-Weighted Decoding (SWD), a training-free strategy designed to filter out premature tokens by penalizing temporal instability. SWD acts as a universal "stability modulator" compatible with any score-based decoding method. In this section, we first establish the information-theoretic foundation, connecting instability to dependency unmasked places (Sec. \ref{sec:method_theory}), then derive the risk-constrained scoring function (Sec. \ref{sec:swd_scoring}), and finally present the implementation (Sec. \ref{sec:implementation}).

\begin{table*}[t]
\centering
\small
\renewcommand{\arraystretch}{1.15} % 稍微增加行高，看着更舒服
\caption{\textbf{Main Results on Parallel Decoding.} Performance comparison of Stability-Weighted Decoding (SWD) against standard baselines across three representative scoring metrics: Confidence, Margin, and Entropy. Base denotes the vanilla EB Sampler strategy, and +SWD denotes our method with $\lambda=5$. SWD consistently improves Accuracy (\textbf{Acc}) while maintaining high efficiency (lower \textbf{NFE}). Best Accuracy is marked in \textbf{bold}.}
\label{tab:main_results}
\resizebox{\textwidth}{!}{
\begin{tabular}{llcccccccc}
\toprule
\multirow{2}{*}{\textbf{Metric}} & \multirow{2}{*}{\textbf{Method}} & \multicolumn{2}{c}{\textbf{HumanEval}} & \multicolumn{2}{c}{\textbf{MBPP}} & \multicolumn{2}{c}{\textbf{GSM8K}} & \multicolumn{2}{c}{\textbf{MATH500}} \\
\cmidrule(lr){3-4} \cmidrule(lr){5-6} \cmidrule(lr){7-8} \cmidrule(lr){9-10}
& & \textbf{Acc} $\uparrow$ & \textbf{NFE} $\downarrow$ & \textbf{Acc} $\uparrow$ & \textbf{NFE} $\downarrow$ & \textbf{Acc} $\uparrow$ & \textbf{NFE} $\downarrow$ & \textbf{Acc} $\uparrow$ & \textbf{NFE} $\downarrow$ \\
\midrule

% ==================== PANEL A: LLaDA ====================
% 使用背景色行作为 Panel 标题，替代 Multirow Model
\multicolumn{10}{c}{\cellcolor{graybg}\textbf{Panel A: LLaDA-8B}} \\
\midrule

% --- Confidence ---
\multirow{2}{*}{Confidence} 
  & Base & 28.0 & 102.9 & 23.4 & 110.3 & 31.5 & 224.1 & 19.6 & 226.6 \\
  & \textbf{+ SWD} & \textbf{36.0} & 77.2 & \textbf{34.8} & 66.8 & \textbf{67.3} & 163.2 & \textbf{26.2} & 205.7 \\
\midrule
% --- Margin ---
\multirow{2}{*}{Margin} 
  & Base & 23.2 & 100.8 & 26.0 & 108.3 & 38.7 & 215.8 & 21.8 & 221.0 \\
  & \textbf{+ SWD} & \textbf{29.9} & 66.4 & \textbf{34.0} & 65.9 & \textbf{64.4} & 166.9 & \textbf{25.8} & 198.4 \\
\midrule

% --- Entropy ---
\multirow{2}{*}{Entropy} 
  & Base & 28.7 & 108.9 & 18.4 & 114.1 & 26.3 & 226.6 & 19.6 & 230.7 \\
  & \textbf{+ SWD} & \textbf{32.3} & 103.5 & \textbf{26.6} & 108.0 & \textbf{48.7} & 191.8 & \textbf{20.4} & 224.3 \\

\midrule
% ==================== PANEL B: DREAM ====================
\multicolumn{10}{c}{\cellcolor{graybg}\textbf{Panel B: Dream-7B}} \\
\midrule

% --- Confidence ---
\multirow{2}{*}{Confidence} 
  & Base & 12.2 & 179.7 & 29.4 & 212.4 & \textbf{29.3} & 286.6 & 10.8 & 364.1 \\
  & \textbf{+ SWD} & \textbf{24.4} & 67.4 & 26.2 & 86.2 & \textbf{61.6} & 119.3 & \textbf{35.0} & 151.6 \\
\midrule

% --- Margin ---
\multirow{2}{*}{Margin} 
  & Base & 10.4 & 173.4 & \textbf{23.4} & 202.6 & 29.5 & 275.6 & 11.4 & 354.9 \\
  & \textbf{+ SWD} & \textbf{14.6} & 63.9 & 17.6 & 69.9 & \textbf{61.0} & 111.5 & \textbf{33.8} & 135.4 \\
\midrule

% --- Entropy ---
\multirow{2}{*}{Entropy} 
  & Base & 21.3 & 188.4 & 33.2 & 174.7 & 28.4 & 292.3 & 8.4 & 382.4 \\
  & \textbf{+ SWD} & \textbf{40.9} & 84.9 & \textbf{42.6} & 109.4 & \textbf{65.0} & 132.1 & \textbf{38.6} & 181.6 \\

\bottomrule
\end{tabular}
}
\end{table*}

\subsection{Theoretical Framework: Instability as Dependency}
\label{sec:method_theory}

Our core insight is that the temporal instability of a token serves as a rigorous proxy for its entanglement with the remaining masked context. We formalize this intuition through the following theorem.

\begin{theorem}[\textbf{Sensitivity-Dependency Bound}]
\label{thm:bound}
Let $\mathcal{D}_{\text{temp}}^{(i)} = D_{\text{KL}}(p_\theta^i(\cdot|\mathbf{x}_t) \parallel p_\theta^i(\cdot|\mathbf{x}_{t+1}))$ be the temporal instability of token $i$ at step $t$. The expected instability is a strict lower bound on the mutual information between the token and the unrevealed masked context $\mathbf{U}_{t+1}$:
\begin{equation}
    I(x_0^i; \mathbf{U}_{t+1} \mid \mathbf{x}_{t+1}) \ge \mathbb{E}\left[ \mathcal{D}_{\text{temp}}^{(i)} \right].
    \label{eq:dependency_bound}
\end{equation}
\end{theorem}
\begin{proof}
    See Appendix \ref{app:theoretical_proofs}.
\end{proof}

\textbf{Implication.} Theorem \ref{thm:bound} establishes a necessary condition for independence: if a token is unstable (high KL divergence), it is mathematically guaranteed to be dependent on the masked context. Unmasking such a token is premature, as it has not yet decoupled from the latent structure of the sequence. Consequently, delaying unmasking until tokens have stabilized is essential for improving generation quality.

\subsection{Stability-Weighted Scoring}
\label{sec:swd_scoring}

Building upon Theorem \ref{thm:bound}, we identify temporal instability is a quantifiable lower bound on the dependency of masked context. To ensure safe decoding, we reformulate the token selection process as a risk-constrained optimization problem.

\textbf{Universal Formulation.}
To maintain generality, let $\mathcal{S}_{\text{base}}(i)$ denote \textit{any} arbitrary utility score provided by an existing decoding policy (e.g., Confidence, Margin, or Entropy). Our objective is to maximize the cumulative utility of the selected subset $S \subset M_t$, subject to the constraint that the aggregate dependency risk remains within a safety budget $\lambda$.

The optimization problem is formulated as follows:
\begin{equation}
    \max_{S} \sum_{i \in S} \log \mathcal{S}_{\text{base}}(i) \quad \text{s.t.} \quad \sum_{i \in S} \mathcal{D}_{\text{temp}}^{(i)} \le \lambda,
\end{equation}
where the constraint effectively bounds the total Mutual Information $I(S; \mathbf{U}_{t+1})$ according to Theorem \ref{thm:bound}.

\begin{figure*}[t]
    \centering
    \includegraphics[width=0.85\linewidth]{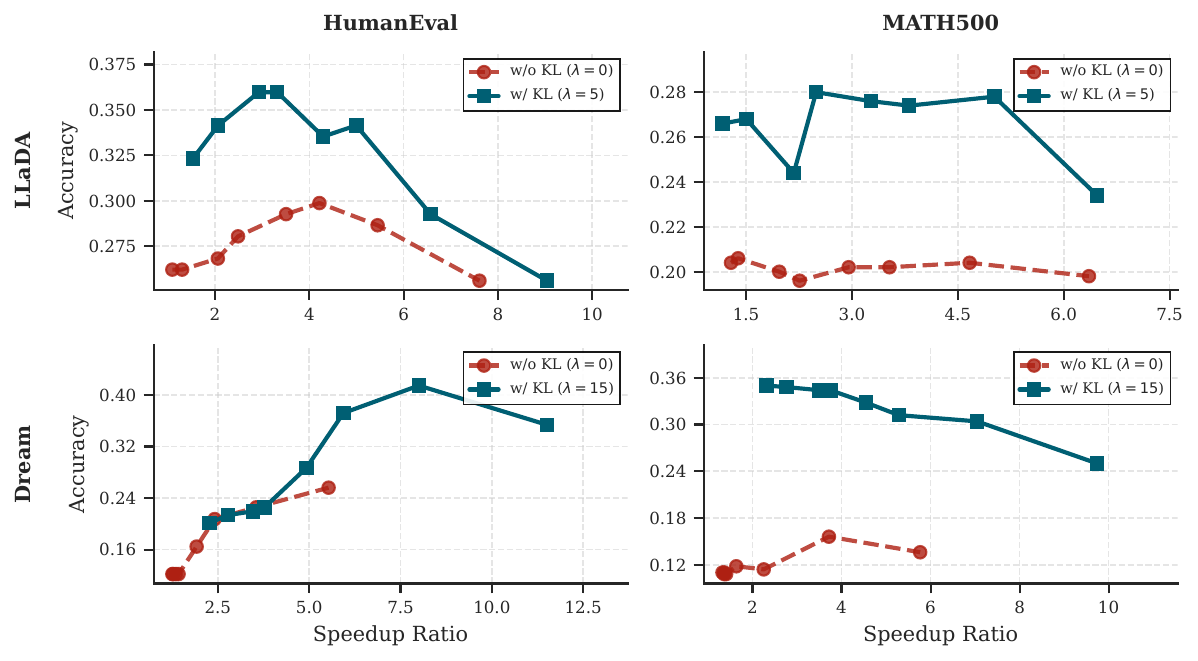}
    \caption{\textbf{Efficiency-Accuracy Trade-off.} We evaluate generation quality (Accuracy) against inference speed (Speedup Ratio) by varying the uncertainty budget $\gamma$ on HumanEval and MATH500. 
    The Solid Blue Line represents our SWD-enhanced decoding, while the Dashed Red Line denotes the standard Confidence baseline. 
    SWD consistently perform well, achieving significantly higher accuracy at equivalent or greater speedup levels compared to the baseline, particularly effectively preventing performance collapse in high-acceleration regimes.}
    \label{fig:efficiency}
\end{figure*}

\textbf{The SWD Score.}
By applying the method of Lagrange multipliers, we relax the constrained decoding problem into an unconstrained maximization of the Lagrangian:

\begin{equation}
    \mathcal{L}(S) = \sum_{i \in S} \left( \log \mathcal{S}_{\text{base}}(i) - \lambda \cdot \mathcal{D}_{\text{temp}}^{(i)} \right),
\end{equation}
where $\lambda \ge 0$ controls the penalty imposed on unstable tokens.

% Because the objective decomposes over tokens, maximizing 
% $\mathcal{L}(S)$ reduces to selecting tokens that contribute most to the overall score.
% Based on this observation, we exponentiate the per-token term and define the resulting quantity as the Stability-Weighted Decoding (SWD) score.

Since the objective decomposes into independent token-level terms, maximizing $\mathcal{L}(S)$ can be performed by evaluating each token individually.
We therefore exponentiate the per-token contribution and define the resulting quantity as the Stability-Weighted Decoding (SWD) score:
\begin{equation}
    \mathcal{S}_{\text{SWD}}(i) = \underbrace{\mathcal{S}_{\text{base}}(i)}_{\text{Base Utility}} \cdot \underbrace{\exp\left( -\lambda \cdot \mathcal{D}_{\text{temp}}^{(i)} \right)}_{\text{Stability Modulator}}.
    \label{eq:swd_score}
\end{equation}

Eq.~(\ref{eq:swd_score}) shows that SWD acts as a multiplicative modulation over an arbitrary decoding score.
Tokens with large temporal divergence are exponentially down-weighted, while stable tokens retain their original scores.
As a result, SWD down-weights tokens that are unstable and dependent on masked context, without requiring any modification or retraining of the underlying model and existing decoding method.

\subsection{Implementation}
\label{sec:implementation}

SWD is straightforward to implement and requires no modifications to the model architecture or the underlying decoding procedure. Algorithm \ref{alg:swd} shows how to use SWD in Top-1 decoding as an example. 
It can be integrated into existing decoding methods with minimal changes to the codebase.
Figure~\ref{fig:code_comparison} illustrates a code by comparing a standard confidence-based decoding step with its SWD-augmented variant.
By caching the logits from the previous iteration, SWD applies a stability-based modulation to the confidence score, with negligible additional computational overhead.

\begin{figure*}[t] % 使用 figure* 跨双栏，[t] 置顶
    \centering
    % 宽度设置为 \textwidth (页面文本宽度)，这样图片会自动撑满两栏
    \includegraphics[width=\textwidth]{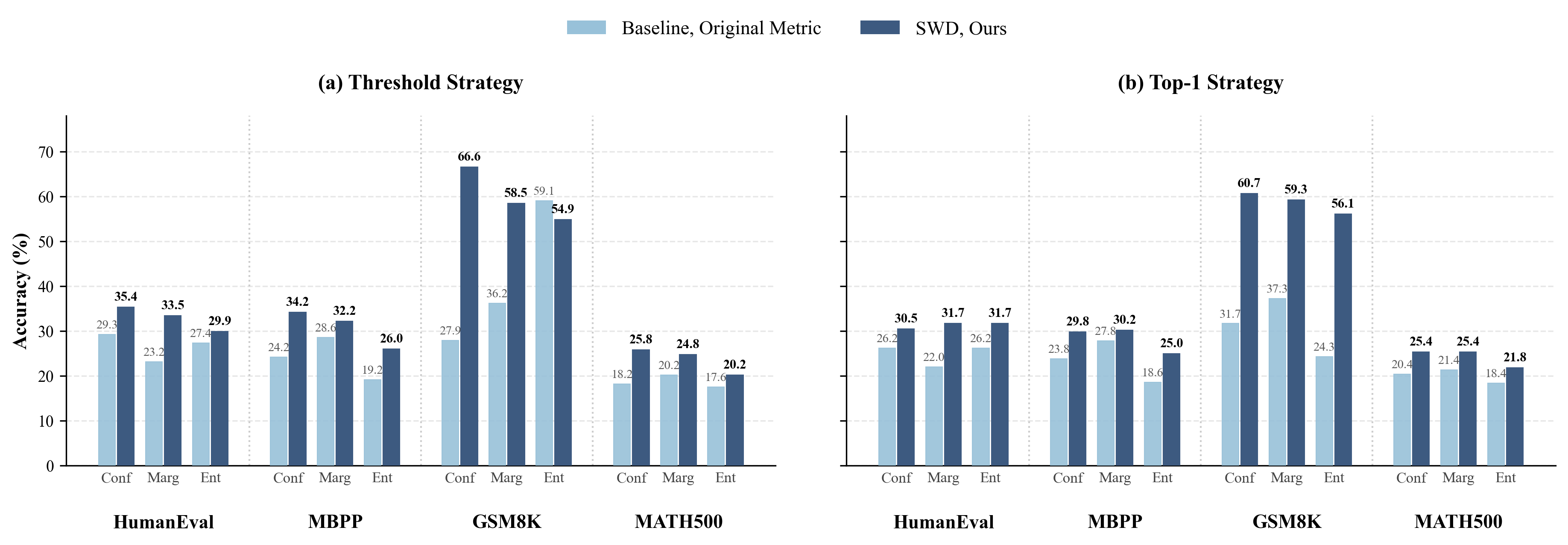}
    
    % 调整图片与 Caption 之间的距离 (根据需要微调)
    % \vspace{-15pt} 
    
    \caption{\textbf{Detailed selection strategies analysis.} Comparisons across selection policies: \textbf{(a) Threshold Strategy} and \textbf{(b) Top-1 Strategy}. SWD Blue consistently outperforms the Baseline Gray.}
    \label{fig:threshold_topk}
\end{figure*}

\section{Experiments}
\label{sec:experiments}

In this section, we present a comprehensive empirical evaluation of Stability-Weighted Decoding (SWD). Our experiments are designed to validate SWD as a generalized decoding mechanism that seamlessly integrates with existing frameworks.
Specifically, we evaluate SWD across multiple state-of-the-art models and benchmarks, demonstrating its universal compatibility with arbitrary scoring functions and diverse selection strategies.
Furthermore, we examine the robustness of SWD under varying decoding settings, such as different block sizes in Block Diffusion and varying hyperparameters for selection policies to assess the accuracy-efficiency trade-off in terms of the number of function evaluations (NFE), confirming SWD's effectiveness across various acceleration ratios. 
Finally, we conduct ablation studies to analyze the contribution of hyperparameter $\lambda$.

\subsection{Experimental Setup}
\label{sec:exp_setup}

We evaluate the effectiveness of SWD on two state-of-the-art diffusion language models: LLaDA-8B-Instruct~\cite{llada} and Dream-v0-Instruct-7B~\cite{dream}. 
Our evaluation covers four representative benchmarks: HumanEval~\cite{humaneval} and MBPP~\cite{mbpp} for code generation, and GSM8K~\cite{gsm8k} and MATH500~\cite{math500} for mathematical reasoning. 
Following prior work~\cite{li2025beyond}, the generation length is fixed to 256 for code tasks and 512 for reasoning tasks. 
To demonstrate universality, we compare SWD against three standard scoring metrics: Confidence, Margin, and Entropy.
Unless otherwise specified, all methods employ the EB-Sampler~\cite{ebsampler} as the default selection policy with an uncertainty budget of $\gamma=0.1$. 
For our method, we report results using the optimal penalty strength $\lambda$ identified in the ablation study (see Section~\ref{sec:ablation_studies}).
All experiments were performed on 4 NVIDIA A100 GPUs.

\subsection{Universal Enhancement across Scoring Metrics}
\label{sec:main_results}

% We first verify SWD as a universal modulator compatible with arbitrary scoring functions. Table \ref{tab:main_results} presents the performance of SWD integrated into Confidence, Margin, and Entropy baselines in two models, LLaDA and Dream.
% As shown in Table \ref{tab:main_results}, SWD consistently enhances accuracy across all models, metrics, and datasets, you can find a more detialed version in Table \ref{tab:app_llada_full} and Table \ref{tab:app_dream_full}.
% SWD consistently enhances generation quality across diverse metrics and benchmarks. The gains are particularly substantial on the Dream-7B model. For instance, on the rigorous MATH500 benchmark, SWD boosts the standard Confidence baseline from 10.8\% to 35.0\%, effectively rectifying the model's tendency to hallucinate under high uncertainty. 
% Similarly, on LLaDA-8B, SWD yields robust improvements (e.g., +8.0\% on HumanEval with Confidence), demonstrating its effectiveness even on stronger base models.
% Crucially, SWD's benefits are not confined to confidence-based decoding. It delivers comparable boosts to Margin and Entropy metrics. This empirically validates our theoretical premise: temporal stability serves as an orthogonal and complementary signal to static plausibility, capable of refining any scoring function without task-specific tuning.
% While SWD achieves superior accuracy in most cases, we observe a performance drop in the Dream-MBPP setting with Confidence and Margin metrics. However, this comes with a drastic reduction in computational cost (e.g., Margin NFE drops from 158.6 to 74.2).

We verify SWD as a universal modulator compatible with arbitrary scoring functions. Table \ref{tab:main_results} presents the performance comparison, with detailed breakdowns provided in Tables \ref{tab:app_llada_full} and \ref{tab:app_dream_full}. SWD consistently enhances generation quality across diverse metrics and benchmarks. 
The gains are particularly substantial on Dream-7B, where SWD boosts the Confidence baseline on MATH500 from 10.8\% to 35.0\%, effectively rectifying the model's tendency to hallucinate under high uncertainty. Robust improvements are also observed on the stronger LLaDA-8B model (e.g., +8.0\% on HumanEval), confirming that SWD scales effectively with model capability. Crucially, these benefits extend to Margin and Entropy metrics, empirically validating that temporal stability serves as an orthogonal signal capable of refining any static plausibility score. 
While SWD achieves superior accuracy in most cases, we observe a slight performance drop in the Dream-MBPP setting. However, this is offset by a drastic reduction in computational cost, with Margin NFE dropping from 202.6 to 69.9. 
By recalibrating $\gamma$ from 0.1 to 0.005, SWD improves to 26.6\% Acc and 98.6 NFE, exceeds the baseline performance.

\subsection{Comparision with History-Aware Decoding}
We compare our method(we use confidence as score combined with SWD, EB-Sampler for selection) with other history-aware decoding method, KLASS\cite{klass} and CreditDecoding\cite{creditdecoding}. KLASS uses confidence as score with a KL-threshold; CreditDecoding uses history-augmented scores.
Besides, both method are score-based method, so SWD can also further modulate them. We use LLaDA for experiment and strictly followed the optimal hyperparameters reported in the original papers: for CreditDecoding, we used $\{\alpha=0.65, \beta=0.7, \gamma=0.65, \tau=0.9\}$; for KLASS, we used $\{\tau_{KL}=0.001, \tau_{Conf}=0.9\}$. We also did internal ablation sweeps to verify these settings. The result is provided in Table \ref{tab:history_aware_compact}.
combined with SWD significantly boosting its performance. 
CreditDecoding shows marginal gains in Block 32 due to information redundancy (both creditdecoding and SWD use stability signals, more detailed discussion is in Appendix B.4), but SWD still provides improvements in Block Full scenarios.

\begin{table}[htbp]
\centering
\caption{Comparison with History-Aware Baselines. Our method achieve best performance and SWD consistently improves average accuracy (Avg) across different decoding settings.}
\label{tab:history_aware_compact}
\vskip 0.1in
\begin{small}
\setlength{\tabcolsep}{3.5pt} % 稍微收紧列间距
\begin{tabular}{lcccccc}
\toprule
\textbf{Strategy} & \textbf{$\lambda$} & \textbf{HEval} & \textbf{MBPP} & \textbf{GSM8K} & \textbf{MATH} & \textbf{Avg} \\
\midrule
\multicolumn{7}{l}{\textit{Full Block}} \\
\midrule
Ours & 5.0 & 36.00 & 34.80 & 67.30 & 36.00 & 43.53 \\
\cmidrule(lr){1-7}
KLASS & 0.0 & 28.05 & 24.00 & 58.83 & 18.40 & 32.32 \\
\quad + SWD & 5.0 & 31.70 & 29.00 & 59.96 & 25.60 & 36.57 \\
\cmidrule(lr){1-7}
CreditDecoding & 0.0 & 29.27 & 24.20 & 59.44 & 18.20 & 32.78 \\
\quad + SWD & 5.0 & 33.54 & 30.40 & 60.35 & 26.80 & 37.77 \\
\midrule
\multicolumn{7}{l}{\textit{Block 32}} \\
\midrule
Ours & 1.0 & 42.70 & 38.80 & 82.90 & 39.00 & 50.85 \\
\cmidrule(lr){1-7}
KLASS & 0.0 & 39.63 & 36.80 & 79.15 & 39.60 & 48.80 \\
\quad + SWD & 1.0 & 41.46 & 35.20 & 79.15 & 39.80 & 48.90 \\
\cmidrule(lr){1-7}
CreditDecoding & 0.0 & 38.41 & 39.20 & 79.08 & 40.40 & 49.27 \\
\quad + SWD & 1.0 & 39.63 & 37.60 & 78.62 & 39.20 & 48.76 \\
\bottomrule
\end{tabular}
\end{small}
\end{table}

% \subsection{Robustness across Acceleration Regimes}
% \label{sec:efficiency}

% A critical challenge in parallel decoding is the "speed-accuracy trade-off": aggressive acceleration (lower NFE) typically degrades generation quality. We evaluate SWD's robustness under these aggressive settings by varying the uncertainty budget $\gamma$ in the EB-Sampler.

% Figure \ref{fig:efficiency} illustrates the efficiency curve of two models on the HumanEval and Math500 benchmark. 

% SWD (Blue Line) consistently lies above the Baseline (Red Dashed Line), effectively pushing the Pareto frontier. This implies that for any given computational budget (NFE), SWD yields higher accuracy.

\begin{table}[ht] % 去掉 *，变为单栏
    \centering
    \caption{\textbf{Main Results on Semi-Autoregressive Decoding (LLaDA, Block=32).} 
    Comparison of Accuracy (\textbf{Acc} $\uparrow$) across three scoring metrics. \textbf{Base} denotes the standard EB-Sampler strategy, and \textbf{+ SWD} denotes our proposed method. SWD consistently improves performance on complex reasoning tasks while maintaining a compact decoding footprint.}
    \label{tab:main_results_block32}
    
    \renewcommand{\arraystretch}{1.1} % 稍微调整行高
    \setlength{\tabcolsep}{3pt}       % 调整列间距以适应单栏宽度
    
    \begin{small} % 使用较小字体确保在单栏内不拥挤
        \begin{tabular}{l l c c c c c}
            \toprule
            \textbf{Metric} & \textbf{Method} & \textbf{HEval} & \textbf{MBPP} & \textbf{GSM8K} & \textbf{MATH} & \textbf{Avg} \\
            \midrule
            
            % ================= Confidence Group =================
            \multirow{2}{*}{Confidence} 
              & Base  & 39.0 & 37.8 & \textbf{83.5} & 38.2 & 49.63\\
              & \textbf{+ SWD} & \textbf{42.7} & \textbf{38.8} & 82.9 & \textbf{39.0} & \textbf{50.85}\\
            \midrule
            
            % ================= Margin Group =================
            \multirow{2}{*}{Margin} 
              & Base  & 40.9 & \textbf{38.4} & \textbf{82.3} & 38.6 & 47.55 \\
              & \textbf{+ SWD} & \textbf{41.5} & 37.4 & \textbf{82.3} & \textbf{38.8} & \textbf{50.00} \\
            \midrule

            % ================= Entropy Group =================
            \multirow{2}{*}{Entropy} 
              & Base  & 42.7 & \textbf{38.2} & \textbf{83.5} & 37.0 & 50.35\\
              & \textbf{+ SWD} & \textbf{43.9} & 38.0 & 82.5 & \textbf{39.6} & \textbf{51.00} \\
            
            \bottomrule
        \end{tabular}
    \end{small}
\end{table}

\subsection{Robustness across Acceleration Regimes}
\label{sec:robustness}

A critical challenge in parallel decoding is balancing inference speed with generation quality. To rigorously evaluate the robustness of SWD, we simulate a wide spectrum of decoding speeds by sweeping the uncertainty budget $\gamma$ in the EB-Sampler across the set $\{0.0005, 0.005, 0.05, 0.1, 0.5, 1, 2\}$. 
A larger $\gamma$ increases the tolerance for uncertainty per step, resulting in higher speedup ratios but a greater risk of quality degradation. 

As illustrated in Figure \ref{fig:efficiency}, which compares the standard Confidence baseline against Confidence + SWD, our method consistently yields a superior trade-off profile, with the SWD curve lying strictly above the baseline across all configurations. 
Notably, we observe distinct behaviors between models: while the LLaDA baseline initially exhibits higher accuracy, it suffers from rapid degradation as acceleration scales up, both baseline and SWD enhanced method dropping to approximately 25\% on HumanEval and 20\% on MATH500 under aggressive settings. 
Dream enhanced by SWD demonstrates exceptional robustness, not only overtaking LLaDA's performance but also sustaining high generation quality—maintaining $\sim$30\% accuracy on HumanEval even at significantly higher speedup ratios where LLaDA falters. 
These results confirm that explicitly penalizing temporal instability effectively mitigates the performance collapse associated with aggressive parallel decoding.

\subsection{Generalizatin to Selection Strategies}
\label{sec:Generalization}

To assess the generality of SWD beyond the EB-Sampler selection method, we evaluate it under diverse token selection paradigms.
We consider two selection policies: a Threshold Strategy, which unmasks all tokens whose scores exceed a threshold $\tau$, and a Top-1 Strategy, which selects only the highest-scoring token at each step, using confidence, margin, and entropy-based scores.
As shown in Figure \ref{fig:threshold_topk}, SWD consistently outperforms the corresponding baselines across all benchmarks and scoring metrics except entropy-GSM8K setting (see Tables \ref{tab:app_llada_threshold} and \ref{tab:app_llada_top1} for detailed results). The improvements are especially pronounced on reasoning-intensive tasks; for example, on GSM8K under the Threshold strategy, SWD increases accuracy from approximately 27.9\% to over 66.6\%, and nearly doubles performance under the Top-1 setting. These results indicate that temporal stability provides a robust and universally beneficial signal, independent of the specific selection policy or scoring metric.

We further examine SWD in semi-autoregressive decoding by applying it to LLaDA with a fixed block size of $K=32$. As summarized in Table \ref{tab:main_results_block32}, SWD improves average performance across most configurations, with consistent gains for Confidence and NegEntropy scores, including substantial improvements on challenging benchmarks such as HumanEval (+3.7\%). While Margin-based decoding shows comparable performance and saturated benchmarks like GSM8K exhibit limited variation, the overall trend confirms that penalizing temporal instability helps prevent premature token commitment even under block-wise generation constraints.

\subsection{Ablation Studies}
\label{sec:ablation_studies}
% ================= FIGURE/TABLE FOR LAMBDA =================
\begin{figure}[h]
    \centering
    % Placeholder: Replace with your actual plot
    \includegraphics[width=1.0\linewidth]{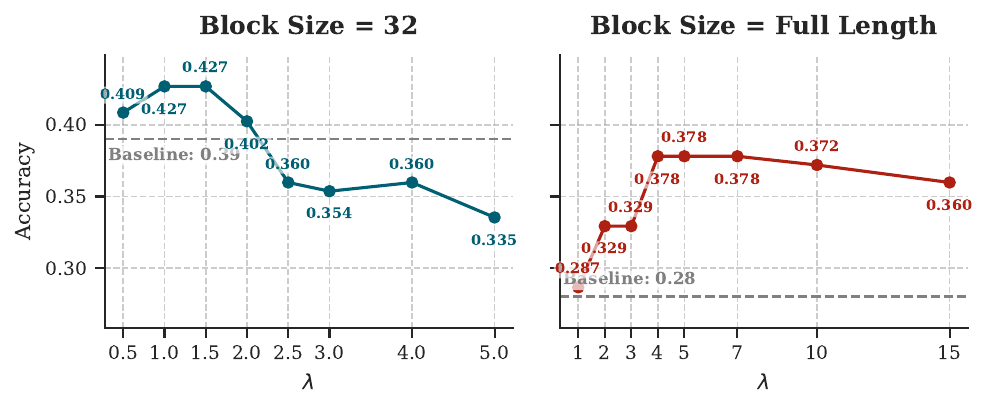}
    \caption{\textbf{Sensitivity to $\lambda$.} The grey dashed line represents the Confidence baseline ($\lambda=0$). SWD outperforms the baseline across a broad range of $\lambda$, demonstrating robustness.}
    \label{fig:lambda_sensitivity}
\end{figure}

Finally, we analyze the sensitivity of the hyperparameter $\lambda$, which balances the trade-off between absolute confidence and temporal stability. 
All experiments are conducted on LLaDA-8B with Confidence scoring across two decoding settings: Block-32 and Full-Sequence (EB-Sampler with $\gamma=0.1$) in HumanEval.
As shown in Figure \ref{fig:lambda_sensitivity}, both decoding settings exhibit an inverted U-shaped trend, indicating that $\lambda$ requires tuning to reach the "sweet spot": a value too small neglects the stability signal, while a value too large over-penalizes valid tokens, causing the model to ignore high-confidence candidates.

Specifically, the Block-32 setting displays a distinct peak around $\lambda \approx 1.5$, where performance maximizes before dropping, as the local block constraint already provides partial stability. In contrast, the Full-Sequence setting benefits from a much broader optimal range, consistently outperforming the baseline. This suggests that in fully parallel decoding, where hallucination risks are higher, a stronger stability enforcement is universally beneficial.

\begin{table}[htbp]
\centering
\caption{Efficiency and Resource Consumption Analysis. $T_m$ and $T_{kl}$ denote model forward time and KL-divergence computation time (in ms), respectively. NFE refers to the number of function evaluations. VRAM is measured in GB.}
\label{tab:efficiency_analysis}
\vskip 0.1in
\begin{small}
\setlength{\tabcolsep}{1pt} % 收紧列间距以适应单栏
\begin{tabular}{llcccccccc}
\toprule
\textbf{Len} & \textbf{Strategy} & \textbf{Acc} & \textbf{NFE} & \textbf{$T_{m}$} & \textbf{$T_{kl}$} & \textbf{$T_{total}$} & \textbf{Avg V} & \textbf{Peak V} \\
\midrule
\multirow{2}{*}{256} & Baseline & 23.04 & 102.90 & 7745.9 & 0.00 & 7745.9 & 17.67 & 23.11 \\
 & \textbf{SWD} & 36.00 & 75.70 & 5976.6 & 193.4 & 6170.1 & 18.35 & 24.24 \\
\midrule
\multirow{2}{*}{512} & Baseline & 15.86 & 149.18 & 16374.0 & 0.00 & 16374.0 & 19.96 & 27.00 \\
 & \textbf{SWD} & 28.04 & 82.78 & 9148.2 & 419.6 & 9567.8 & 21.11 & 28.42 \\
\bottomrule
\end{tabular}
\end{small}
\end{table}

\subsection{Efficiency and Resource Consumption Analysis}
We evaluate the computational efficiency and memory overhead of SWD using LLaDA-8B ($|V| = 126,464$) on HumanEval, with generation lengths set to 256 and 512 ($B=L$). As shown in Table \ref{tab:efficiency_analysis}, while SWD introduces a marginal KL computation overhead ($T_{kl}$), it significantly reduces the total NFE. Since $T_{kl}$ accounts for only 3--4\% of the total latency, the time saved by avoiding redundant model forward passes ($T_m$) far outweighs the local computation cost. 
This confirms that SWD remains highly efficient as sequence length increases, trading minimal local overhead for substantial global speedup.The VRAM overhead of SWD is also practical for 8B-parameter models. For $L=256$, SWD incurs a negligible increase in average and peak VRAM ($\sim$4--5\% over the baseline). This footprint is comparable to other history-aware methods, such as CreditDecoding, which maintains a logit trace matrix $C \in \mathbb{R}^{L \times |V|}$, or KLASS, which also tracks historical KL divergence. 
Given the significant reduction in NFE and end-to-end latency, the minor memory trade-off is well-justified for practical deployment.
\section{Conclusion}
\label{sec:conclusion}

In this paper, we propose Stability-Weighted Decoding (SWD), a training-free and plug-and-play approach for reducing premature commitment in diffusion language model decoding.
We show that temporal instability provides a principled signal of a token’s dependency on masked context and formalize this relationship through the Sensitivity-Dependency Bound.
Without modifying the diffusion process or retraining the model, SWD consistently improves generation quality when integrated with a wide range of score-based decoding strategies.
Experiments on code generation and mathematical reasoning tasks demonstrate that explicitly accounting for temporal stability leads to more reliable parallel decoding across diverse decoding regimes.
We believe that stability-aware decoding offers a simple yet general direction for improving diffusion-based generation, and may extend naturally to multimodal settings or more flexible re-masking schemes.

\section*{Acknowledgements}
This research was conducted using the computing resources provided by the Research Centre for the Mathematical Foundations of Generative AI (PO046811) at The Hong Kong Polytechnic University.

% In this paper, we introduced \textbf{Stability-Weighted Decoding (SWD)}, a training-free, plug-and-play mechanism designed to mitigate premature commitment in the generation of diffusion language models. 
% By identifying temporal instability as a rigorous proxy for latent dependency on the masked context, we theoretically established the Sensitivity-Dependency Bound, providing a mathematical foundation for safe unmasking. 
% Empirically, SWD acts as a universal modulator that seamlessly enhances arbitrary score-based decoding policies (Confidence, Margin, Entropy) without requiring model retraining. 
% Extensive experiments across code generation and mathematical reasoning benchmarks demonstrate that SWD not only consistently improves generation quality but also exhibits exceptional robustness under aggressive acceleration regimes. 
% By effectively suppressing statistically plausible but temporally unstable tokens, SWD pushes the efficiency-accuracy frontier of dLLMs. 
% Future work may explore extending this stability-aware paradigm to multimodal diffusion models or integrating it with iterative re-masking strategies to further close the gap with autoregressive baselines.

% \section*{Impact Statements}
% This paper presents work whose goal is to advance the field of diffusion language model. There are many potential societal consequences of our work, none of which we feel must be specifically highlighted here.

\bibliography{example_paper}
\bibliographystyle{icml2026}

\newpage
\appendix
\onecolumn
\section{Appendix A: Information-Theoretic Foundation of Stability-Weighted Decoding}
\label{app:theoretical_proofs}

In this section, we provide a rigorous information-theoretic justification for SWD. We demonstrate that the observed temporal instability (KL divergence) serves as a strict lower bound on the token's dependency on the remaining masked context. Consequently, minimizing instability is mathematically equivalent to minimizing the risk of latent dependency and potential hallucinations.

\subsection{Preliminaries and Notation}

We model the diffusion generation process as a sequence of refining contexts $\mathbf{x}_T, \dots, \mathbf{x}_t, \dots, \mathbf{x}_0$.
\begin{itemize}
    \item $x_0^i$: The ground-truth token at position $i$.
    \item $\mathbf{x}_{t+1}$: The context at the previous step $t+1$.
    \item $\mathbf{x}_t$: The context at the current step $t$.
    \item $\mathbf{M}_{t+1}$: The set of indices that are masked at step $t+1$. Let $\mathbf{U}_{t+1} = \{x_0^j \mid j \in \mathbf{M}_{t+1}\}$ be the set of ground-truth values for these masked positions (the "Total Unknowns").
    \item $\Delta \mathbf{I}_t$: The incremental information revealed during the transition $\mathbf{x}_{t+1} \to \mathbf{x}_t$. Note that $\mathbf{x}_t$ can be viewed as the union of $\mathbf{x}_{t+1}$ and $\Delta \mathbf{I}_t$.
\end{itemize}

We define the Temporal Instability $\mathcal{D}_{\text{temp}}^{(i)}$ of token $i$ as the Kullback-Leibler (KL) divergence between its predictive distributions at consecutive steps:
\begin{equation}
    \mathcal{D}_{\text{temp}}^{(i)} \triangleq D_{\text{KL}}\Big( p(x_0^i \mid \mathbf{x}_t) \parallel p(x_0^i \mid \mathbf{x}_{t+1}) \Big).
\end{equation}

\subsection{Equivalence of Instability and Information Gain}

First, we establish that the expected temporal instability is exactly the Conditional Mutual Information (CMI) gained from the context update.

\begin{lemma}
The expected KL divergence for token $x_0^i$ during the update $\mathbf{x}_{t+1} \to \mathbf{x}_t$ equals the mutual information between the token and the updated context $\mathbf{x}_t$, conditioned on the previous context $\mathbf{x}_{t+1}$:
\begin{equation}
    \mathbb{E}_{\mathbf{x}_t \sim p(\cdot \mid \mathbf{x}_{t+1})} \left[ D_{\text{KL}}\Big( p(x_0^i \mid \mathbf{x}_t) \parallel p(x_0^i \mid \mathbf{x}_{t+1}) \Big) \right] = I(x_0^i; \mathbf{x}_t \mid \mathbf{x}_{t+1}).
\end{equation}
\end{lemma}

\begin{proof}
Expanding the definition of expected KL divergence:
\begin{align}
    \mathbb{E}[\mathcal{D}_{\text{temp}}^{(i)}] &= \sum_{\mathbf{x}_t} p(\mathbf{x}_t \mid \mathbf{x}_{t+1}) \sum_{v \in V} p(v \mid \mathbf{x}_t) \log \frac{p(v \mid \mathbf{x}_t)}{p(v \mid \mathbf{x}_{t+1})} \\
    &= \sum_{\mathbf{x}_t} \sum_{v} p(v, \mathbf{x}_t \mid \mathbf{x}_{t+1}) \log \frac{p(v \mid \mathbf{x}_t)}{p(v \mid \mathbf{x}_{t+1})}
\end{align}
Using Bayes' theorem, we substitute $p(v \mid \mathbf{x}_t) = \frac{p(v, \mathbf{x}_t \mid \mathbf{x}_{t+1})}{p(\mathbf{x}_t \mid \mathbf{x}_{t+1})}$ into the logarithm term:
\begin{align}
    \log \frac{p(v \mid \mathbf{x}_t)}{p(v \mid \mathbf{x}_{t+1})} &= \log \frac{p(v, \mathbf{x}_t \mid \mathbf{x}_{t+1})}{p(\mathbf{x}_t \mid \mathbf{x}_{t+1}) p(v \mid \mathbf{x}_{t+1})}
\end{align}
Substituting this back into the summation:
\begin{align}
    \mathbb{E}[\mathcal{D}_{\text{temp}}^{(i)}] &= \sum_{\mathbf{x}_t} \sum_{v} p(v, \mathbf{x}_t \mid \mathbf{x}_{t+1}) \log \frac{p(v, \mathbf{x}_t \mid \mathbf{x}_{t+1})}{p(v \mid \mathbf{x}_{t+1})p(\mathbf{x}_t \mid \mathbf{x}_{t+1})} \\
    &\equiv I(x_0^i; \mathbf{x}_t \mid \mathbf{x}_{t+1}).
\end{align}
\end{proof}

\subsection{The Sensitivity-Dependency Bound}

We now relate this information gain to the token's total dependency on the unknown contexts.

\begin{theorem}
The expected temporal instability of token $x_0^i$ is a strict lower bound on its mutual information with the total masked context $\mathbf{U}_{t+1}$. That is, high instability implies high dependency on the remaining unknowns.
\begin{equation}
    I(x_0^i; \mathbf{U}_{t+1} \mid \mathbf{x}_{t+1}) \ge \mathbb{E}\left[ \mathcal{D}_{\text{temp}}^{(i)} \right].
\end{equation}
\end{theorem}

\begin{proof}
The total set of unknowns at step $t+1$, denoted $\mathbf{U}_{t+1}$, consists of the information revealed in the current step ($\mathbf{x}_t$) and the information remaining masked ($\mathbf{U}_t$).
\begin{equation}
    \text{Information Content}(\mathbf{U}_{t+1}) \equiv \text{Information Content}(\mathbf{x}_t, \mathbf{U}_t)
\end{equation}
By the Chain Rule of Mutual Information, we decompose the total dependency:
\begin{align}
    I(x_0^i; \mathbf{U}_{t+1} \mid \mathbf{x}_{t+1}) &= I(x_0^i; \mathbf{x}_t, \mathbf{U}_t \mid \mathbf{x}_{t+1}) \\
    &= \underbrace{I(x_0^i; \mathbf{x}_t \mid \mathbf{x}_{t+1})}_{\text{Information Gain (Lemma 1)}} + \underbrace{I(x_0^i; \mathbf{U}_t \mid \mathbf{x}_t, \mathbf{x}_{t+1})}_{\text{Residual Dependency}}
\end{align}
From Lemma 1, the first term is exactly $\mathbb{E}[\mathcal{D}_{\text{temp}}^{(i)}]$.
By the non-negativity of mutual information, the second term is non-negative:
\begin{equation}
    I(x_0^i; \mathbf{U}_t \mid \mathbf{x}_t) \ge 0.
\end{equation}
Therefore, we obtain the lower bound:
\begin{equation}
    I(x_0^i; \mathbf{U}_{t+1} \mid \mathbf{x}_{t+1}) \ge I(x_0^i; \mathbf{x}_t \mid \mathbf{x}_{t+1}) = \mathbb{E}\left[ \mathcal{D}_{\text{temp}}^{(i)} \right].
\end{equation}
\end{proof}

% \subsection{Implication: Instability as Evidence of Entanglement}

% The theorem above has profound implications for parallel decoding:
% \begin{enumerate}
%     \item \textbf{Existence of Dependency:} A non-zero KL divergence ($\mathcal{D}_{\text{temp}}^{(i)} \gg 0$) mathematically proves that the total dependency $I(x_0^i; \mathbf{U}_{t+1})$ is non-zero. The token is "entangled" with the masked context.
%     \item \textbf{Risk of Premature Unmasking:} If $\mathcal{D}_{\text{temp}}^{(i)}$ is large, the token is in an active state of information acquisition. Assuming the homogeneity of semantic dependencies in natural language, a token sensitive to the partial update $\mathbf{x}_t$ is statistically likely to remain sensitive to the residual masks $\mathbf{U}_t$.
%     \item \textbf{The Safety Condition:} Ideally, we only unmask a token when it is independent of future context, i.e., $I(x_0^i; \mathbf{U}_{t+1}) \approx 0$. By the contrapositive of Theorem 1, minimizing $\mathcal{D}_{\text{temp}}^{(i)}$ is a necessary condition for independence.
% \end{enumerate}

\subsection{Formulation of Stability-Weighted Decoding}

Based on Theorem 1, we formulate the parallel decoding problem as maximizing confidence subject to a constraint on the total risk (dependency). Let $S$ be the set of tokens selected for unmasking.
\begin{equation}
    \max_{S} \sum_{i \in S} \log p(x_0^i \mid \mathbf{x}_t) \quad \text{s.t.} \quad \sum_{i \in S} \mathcal{D}_{\text{temp}}^{(i)} \le \lambda
\end{equation}
where $\lambda$ is a risk budget representing the maximum tolerable dependency. Using the method of Lagrange multipliers, we relax this into the unconstrained objective utilized in SWD:
\begin{equation}
    \mathcal{L}(S) = \sum_{i \in S} \left( \log p(x_0^i \mid \mathbf{x}_t) - \lambda \cdot \mathcal{D}_{\text{temp}}^{(i)} \right).
\end{equation}
Exponentiating the per-token term yields our proposed scoring function:
\begin{equation}
    \text{SWD}(x_i) = p(x_i \mid \mathbf{x}_t) \cdot \exp\left(-\lambda \cdot \mathcal{D}_{\text{KL}}(p_t^i \parallel p_{t+1}^i)\right).
\end{equation}
This derivation confirms that SWD optimally selects tokens that are both high-confidence and theoretically decoupled from the latent masked context.

\section{Appendix B: Detailed Experimental Results}
\label{app:detailed_results}

In this section, we provide the complete experimental data supporting the main claims of the paper. We present granular performance metrics (Accuracy and NFE) across all models (LLaDA-8B, Dream-7B), decoding settings (Full-Sequence, Block-32), and selection strategies (EB-Sampler, Static, Threshold).

\subsection{LLaDA-8B Comprehensive Results}
\label{app:llada_results}

Table \ref{tab:app_llada_full} presents the extended results for LLaDA-8B under the Full-Sequence setting. We include a wider range of penalty strengths ($\lambda$) and compare against the Static decoding baseline.
Table \ref{tab:app_llada_threshold} details the study on the Threshold-based selection strategy, demonstrating SWD's robustness across different settings. We changed the threshold $\tau$ to make sure NEF are comparable, as SWD changes the value of scores.
Table \ref{tab:app_llada_top1} provides specific results for the Top-1 selection strategy. 
Table \ref{tab:app_llada_block32_detail} provides the granular breakdown of results for LLaDA-8B under the Semi-Autoregressive (Block-32) setting. 
We compare the Static strategy (fixed block generation), the standard EB-Sampler (Base), and SWD-enhanced versions with varying $\lambda$. The data confirms that SWD ($\lambda=1.0$ or $1.5$) consistently boosts performance over the adaptive baseline, particularly in reasoning tasks like HumanEval and MATH500, while maintaining comparable inference costs.

% ================= LLaDA FULL SEQUENCE TABLE =================
\begin{table}[h]
    \centering
    \caption{\textbf{LLaDA-8B Full-Sequence Comprehensive Results.} Performance comparison of different scoring targets (Confidence, Margin, NegEntropy) and strategies (Static, EB-Sampler) with expanded penalty strengths ($\lambda$). \textbf{Acc} is Accuracy (\%) and \textbf{NFE} is Average Number of Function Evaluations. Best accuracy within each target block is bolded.}
    \label{tab:app_llada_full}
    
    \renewcommand{\arraystretch}{1.2}
    \setlength{\tabcolsep}{3.5pt}
    
    \resizebox{\textwidth}{!}{
        \begin{tabular}{l l c c c c c c c c}
            \toprule
            \multirow{2}{*}{\textbf{Target}} & \multirow{2}{*}{\textbf{Strategy}} & \multicolumn{2}{c}{\textbf{HumanEval}} & \multicolumn{2}{c}{\textbf{MBPP}} & \multicolumn{2}{c}{\textbf{GSM8K}} & \multicolumn{2}{c}{\textbf{MATH500}} \\
            \cmidrule(lr){3-4} \cmidrule(lr){5-6} \cmidrule(lr){7-8} \cmidrule(lr){9-10}
             & & \textbf{Acc} $\uparrow$ & \textbf{NFE} $\downarrow$ & \textbf{Acc} $\uparrow$ & \textbf{NFE} $\downarrow$ & \textbf{Acc} $\uparrow$ & \textbf{NFE} $\downarrow$ & \textbf{Acc} $\uparrow$ & \textbf{NFE} $\downarrow$ \\
            \midrule
            
            % ================= Confidence =================
            \multirow{4}{*}{Confidence} 
             & EB (Baseline)   & 28.0 & 102.9 & 23.4 & 110.3 & 31.5 & 224.1 & 19.6 & 226.6 \\
             & + SWD ($\lambda=1.5$) & 32.9 & 91.2  & 27.6 & 100.5 & 45.3 & 195.4 & 21.2 & 213.4 \\
             & + SWD ($\lambda=5.0$) & \textbf{36.0} & 77.2 & \textbf{34.8} & 66.8 & \textbf{67.3} & 163.2 & \textbf{26.2} & 205.7\\
            \midrule
            
            % ================= Margin =================
            \multirow{5}{*}{Margin} 
             & EB (Baseline)   & 23.2 & 100.8 & 26.0 & 108.3 & 38.7 & 215.8 & 21.8 & 221.0 \\
             & + SWD ($\lambda=1.0$) & 23.8 & 92.8  & 27.8 & 103.5 & 39.5 & 203.1 & 20.4 & 214.0 \\
             & + SWD ($\lambda=1.5$) & 23.8 & 88.2  & 28.0 & 99.5  & 42.5 & 195.9 & 21.8 & 209.7 \\
             & + SWD ($\lambda=5.0$) & \textbf{29.9} & 66.4  & \textbf{34.0} & 65.9  & \textbf{64.4} & 166.9 & \textbf{25.8} & 198.4 \\
            \midrule
            
            % ================= Entropy =================
            \multirow{5}{*}{Entropy} 
             & EB (Baseline)   & 28.7 & 108.9 & 18.4 & 114.1 & 26.3 & 226.6 & 19.6 & 230.7 \\
             & + SWD ($\lambda=1.0$) & 28.0 & 106.7 & 21.4 & 112.6 & 27.9 & 221.1 & 19.2 & 230.2 \\
             & + SWD ($\lambda=1.5$) & 29.3 & 106.1 & 22.6 & 111.7 & 29.5 & 216.8 & 19.8 & 229.3 \\
             & + SWD ($\lambda=5.0$) & \textbf{32.3} & 103.5 & \textbf{26.6} & 108.0 & \textbf{48.7} & 191.8 & \textbf{20.4} & 224.3 \\
            % \multirow{1}{*}{Pure KL}
            %  & Only KL & 28.0 & 102.9 & 23.4 & 110.3 & 31.4 & 224.1 & 19.6 & 226.6\\
            \bottomrule
        \end{tabular}
    }
\end{table}

% ================= LLaDA THRESHOLD TABLE =================
\begin{table}[h]
    \centering
    \caption{\textbf{LLaDA on Threshold Strategy.} Performance with varying confidence thresholds ($\tau$) and penalty strengths ($\lambda$). Bold indicates the best result (Highest Acc, Lowest NFE) within each metric group.}
    \label{tab:app_llada_threshold}
    \resizebox{\textwidth}{!}{%
        \begin{tabular}{llcccccccc}
        \toprule
        \multirow{2}{*}{\textbf{Metric}} & \multirow{2}{*}{\textbf{Params}} & \multicolumn{2}{c}{\textbf{HumanEval}} & \multicolumn{2}{c}{\textbf{MBPP}} & \multicolumn{2}{c}{\textbf{GSM8K}} & \multicolumn{2}{c}{\textbf{MATH500}} \\
        \cmidrule(lr){3-4} \cmidrule(lr){5-6} \cmidrule(lr){7-8} \cmidrule(lr){9-10}
         & & \textbf{Acc} & \textbf{NFE} & \textbf{Acc} & \textbf{NFE} & \textbf{Acc} & \textbf{NFE} & \textbf{Acc} & \textbf{NFE} \\
        \midrule
        \multirow{3}{*}{Confidence} 
         & $\tau=0.9, \lambda=0$ & 29.27 & 66.16 & 24.20 & 52.45 & 27.90 & 150.40 & 18.20 & 181.79 \\
         & $\tau=0.8, \lambda=1$ & 31.10 & 61.96 & 26.60 & 42.17 & 38.06 & 131.25 & 18.40 & 172.48 \\
         & $\tau=0.8, \lambda=5$ & \textbf{35.37} & 56.12 & \textbf{34.20} & 34.44 & \textbf{66.64} & 101.55 & \textbf{25.80} & 149.25 \\
        \midrule
        \multirow{3}{*}{Margin} 
         & $\tau=0.9, \lambda=0$ & 23.17 & 65.77 & 28.60 & 60.62 & 36.24 & 167.82 & 20.20 & 195.64 \\
         & $\tau=0.8, \lambda=1$ & 24.39 & 80.59 & 28.60 & 76.05 & 39.35 & 205.15 & 20.60 & 240.16 \\
         & $\tau=0.8, \lambda=5$ & \textbf{33.54} & 53.99 & \textbf{32.20} & 42.26 & \textbf{58.45} & 139.56 & \textbf{24.80} & 176.85 \\
        \midrule
        \multirow{3}{*}{Neg\_Entropy} 
         & $\tau=-0.4, \lambda=0$ & 27.44 & 85.01 & 19.20 & 71.70 & \textbf{59.14} & 90.49 & 17.60 & 199.78 \\
         & $\tau=-0.6, \lambda=1$ & \textbf{30.49} & 79.84 & 22.00 & 58.12 & 57.99 & 88.88 & 17.60 & 194.32 \\
         & $\tau=-0.7, \lambda=5$ & 29.88 & 77.39 & \textbf{26.00} & 59.83 & 54.89 & 81.81 & \textbf{20.20} & 183.36 \\
        \bottomrule
        \end{tabular}%
    }
\end{table}

% ================= LLaDA STATIC TABLE =================
% ================= TABLE B.3: LLaDA TOP-1 (Renamed from Static) =================
\begin{table*}[t]
    \centering
    \caption{\textbf{LLaDA Top-1 Strategy Analysis.} Performance comparison using the \textbf{Top-1 selection strategy} (unmasking only the single highest-scoring token per step). We compare the baseline ($\lambda=0$) against SWD ($\lambda=5$). Note that Top-1 decoding results in fixed NFEs equal to the sequence length (256 or 512).}
    \label{tab:app_llada_top1}
    
    \renewcommand{\arraystretch}{1.2}
    \setlength{\tabcolsep}{6pt}
    
    \resizebox{\textwidth}{!}{%
        \begin{tabular}{llcccccccc}
        \toprule
        \multirow{2}{*}{\textbf{Metric}} & \multirow{2}{*}{\textbf{Method}} & \multicolumn{2}{c}{\textbf{HumanEval}} & \multicolumn{2}{c}{\textbf{MBPP}} & \multicolumn{2}{c}{\textbf{GSM8K}} & \multicolumn{2}{c}{\textbf{MATH500}} \\
        \cmidrule(lr){3-4} \cmidrule(lr){5-6} \cmidrule(lr){7-8} \cmidrule(lr){9-10}
         & & \textbf{Acc} & \textbf{NFE} & \textbf{Acc} & \textbf{NFE} & \textbf{Acc} & \textbf{NFE} & \textbf{Acc} & \textbf{NFE} \\
        \midrule
        \multirow{2}{*}{Confidence} 
         & Top-1 (Base) & 26.2 & 256.0 & 23.8 & 256.0 & 31.7 & 512.0 & 20.4 & 512.0 \\
         & Top-1 + SWD ($\lambda=5$) & \textbf{30.5} & 256.0 & \textbf{29.8} & 256.0 & \textbf{60.7} & 256.0 & \textbf{25.4} & 512.0 \\
        \midrule
        \multirow{2}{*}{Margin}
         & Top-1 (Base) & 22.0 & 256.0 & 27.8 & 256.0 & 37.3 & 512.0 & 21.4 & 512.0 \\
         & Top-1 + SWD ($\lambda=5$) & \textbf{31.7} & 256.0 & \textbf{30.2} & 256.0 & \textbf{59.3} & 256.0 & \textbf{25.4} & 512.0 \\
        \midrule
        \multirow{2}{*}{Entropy}
         & Top-1 (Base) & 26.2 & 256.0 & 18.6 & 256.0 & 24.3 & 512.0 & 18.4 & 512.0 \\
         & Top-1 + SWD ($\lambda=5$) & \textbf{31.7} & 256.0 & \textbf{25.0} & 256.0 & \textbf{56.1} & 256.0 & \textbf{21.8} & 512.0 \\
        \bottomrule
        \end{tabular}%
    }
\end{table*}

% ================= TABLE B.X: LLaDA BLOCK-32 DETAILED =================
\begin{table*}[t]
    \centering
    \caption{\textbf{LLaDA-8B Block-32 Detailed Results.} Comprehensive performance comparison under the semi-autoregressive setting (Block Size = 32). We compare Static decoding, the standard EB-Sampler baseline, and SWD with various penalty strengths ($\lambda$). Bold indicates the best accuracy within each metric group.}
    \label{tab:app_llada_block32_detail}
    
    \renewcommand{\arraystretch}{1.2}
    \setlength{\tabcolsep}{6pt}
    
    \resizebox{\textwidth}{!}{
        \begin{tabular}{l l c c c c c c c c}
            \toprule
            \multirow{2}{*}{\textbf{Metric}} & \multirow{2}{*}{\textbf{Method}} & \multicolumn{2}{c}{\textbf{HumanEval}} & \multicolumn{2}{c}{\textbf{MBPP}} & \multicolumn{2}{c}{\textbf{GSM8K}} & \multicolumn{2}{c}{\textbf{MATH500}} \\
            \cmidrule(lr){3-4} \cmidrule(lr){5-6} \cmidrule(lr){7-8} \cmidrule(lr){9-10}
             & & \textbf{Acc} $\uparrow$ & \textbf{NFE} $\downarrow$ & \textbf{Acc} $\uparrow$ & \textbf{NFE} $\downarrow$ & \textbf{Acc} $\uparrow$ & \textbf{NFE} $\downarrow$ & \textbf{Acc} $\uparrow$ & \textbf{NFE} $\downarrow$ \\
            \midrule
            
            % ================= Confidence =================
            \multirow{4}{*}{Confidence} 
             & Static (Block)   & 39.6 & 256.0 & 38.6 & 256.0 & 82.9 & 512.0 & \textbf{40.4} & 512.0 \\
             & Base (EB)        & 39.0 & 76.6  & 37.8 & 55.9  & \textbf{83.5} & 116.0 & 38.2 & 173.2 \\
             & + SWD ($\lambda=1.0$) & \textbf{42.7} & 76.3  & \textbf{38.8} & 55.1  & 82.9 & 110.6 & 39.0 & 163.3 \\
             & + SWD ($\lambda=1.5$) & \textbf{42.7} & 73.0  & 38.6 & 54.3  & 82.4 & 109.6 & 38.2 & 162.4 \\
            \midrule
            
            % ================= Margin =================
            \multirow{4}{*}{Margin} 
             & Static (Block)   & 40.9 & 256.0 & 37.8 & 256.0 & 82.2 & 512.0 & 38.6 & 512.0 \\
             & Base (EB)        & 40.9 & 74.7  & \textbf{38.4} & 54.2  & 82.3 & 111.2 & 38.6 & 164.4 \\
             & + SWD ($\lambda=1.0$) & \textbf{41.5} & 74.6  & 37.4 & 52.3  & 82.3 & 108.3 & \textbf{38.8} & 159.4 \\
             & + SWD ($\lambda=1.5$) & 40.2 & 72.3  & 38.2 & 51.6  & \textbf{83.2} & 107.2 & 36.8 & 157.6 \\
            \midrule
            
            % ================= NegEntropy =================
            \multirow{5}{*}{NegEntropy} 
             & Static (Block)   & 41.5 & 256.0 & 37.2 & 256.0 & 83.3 & 512.0 & \textbf{40.4} & 512.0 \\
             & Base (EB)        & 42.7 & 80.4  & 38.2 & 58.5  & \textbf{83.5} & 123.5 & 37.0 & 180.6 \\
             & + SWD ($\lambda=0.5$) & \textbf{43.9} & 83.5  & 38.0 & 57.9  & 82.5 & 121.7 & 39.6 & 179.7 \\
             & + SWD ($\lambda=1.0$) & 39.6 & 86.8  & \textbf{38.4} & 59.5  & 82.9 & 119.8 & 37.0 & 176.9 \\
             & + SWD ($\lambda=1.5$) & 38.4 & 87.5  & 35.8 & 60.9  & 82.7 & 118.7 & 39.4 & 174.8 \\
            \bottomrule
        \end{tabular}
    }
\end{table*}

\clearpage

\subsection{Dream-7B Comprehensive Results}
\label{app:dream_results}

Table \ref{tab:app_dream_full} details the results for Dream-7B under the Full-Sequence setting. We observe significant improvements with higher $\lambda$ values (e.g., $\lambda=30$), particularly for the Entropy metric.
Table \ref{tab:app_dream_block32} presents the results for Dream-7B under the Block=32 setting.

% ================= DREAM FULL SEQUENCE TABLE =================
\begin{table*}[h]
    \centering
    \caption{\textbf{Dream-7B Full-Sequence Comprehensive Results.} Comparison of baselines and SWD with various penalty strengths. Note the effectiveness of stronger penalties (e.g., $\lambda=15, 30$) for this model. Bold indicates the best result in each group.}
    \label{tab:app_dream_full}
    \resizebox{\textwidth}{!}{%
        \begin{tabular}{llcccccccc}
        \toprule
        \multirow{2}{*}{\textbf{Metric}} & \multirow{2}{*}{\textbf{Strategy}} & \multicolumn{2}{c}{\textbf{HumanEval}} & \multicolumn{2}{c}{\textbf{MBPP}} & \multicolumn{2}{c}{\textbf{GSM8K}} & \multicolumn{2}{c}{\textbf{MATH500}} \\
        \cmidrule(lr){3-4} \cmidrule(lr){5-6} \cmidrule(lr){7-8} \cmidrule(lr){9-10}
         &  & \textbf{Acc} & \textbf{NFE} & \textbf{Acc} & \textbf{NFE} & \textbf{Acc} & \textbf{NFE} & \textbf{Acc} & \textbf{NFE} \\
        \midrule
        \multirow{6}{*}{Confidence} 
         & EB ($\lambda=0$)    & 12.20 & 179.65 & 29.40 & 212.37 & 29.34 & 286.60 & 10.80 & 364.08 \\
         & + SWD ($\lambda=1.0$)  & 14.02 & 160.84 & 39.40 & 154.71 & 29.04 & 249.50 & 12.20 & 305.93 \\
         & + SWD ($\lambda=1.5$)  & 14.63 & 154.23 & 41.40 & 152.61 & 31.24 & 224.87 & 13.80 & 277.19 \\
         & + SWD ($\lambda=5.0$)  & 17.07 & 73.56 & \textbf{44.40} & 87.57 & 52.69 & 123.84 & 25.80 & 115.06 \\
         & + SWD ($\lambda=15.0$) & 22.56 & 67.73 & 26.60 & 86.81 & 59.29 & 116.94 & 34.40 & 136.58 \\
         & + SWD ($\lambda=30.0$) & \textbf{24.39} & 67.41 & 26.20 & 86.24 & \textbf{61.64} & 119.26 & \textbf{35.00} & 151.64 \\
        \midrule
        \multirow{3}{*}{Margin} 
         & EB ($\lambda=0$)    & 10.37 & 173.40 & \textbf{23.40} & 202.60 & 29.49 & 275.62 & 11.40 & 354.94 \\
         & + SWD ($\lambda=15$)   & 13.41 & 60.87 & 14.00 & 74.24 & 58.07 & 105.94 & \textbf{34.40} & 109.84 \\
         & + SWD ($\lambda=30$)   & \textbf{14.63} & 63.93 & 17.60 & 69.93 & \textbf{61.03} & 111.53 & 33.80 & 135.41 \\
        \midrule
        \multirow{3}{*}{Neg\_Entropy} 
         & EB ($\lambda=0$)    & 21.34 & 188.44 & 33.20 & 174.73 & 28.36 & 292.32 & 8.40 & 382.38 \\
         & + SWD ($\lambda=15$)   & 28.05 & 154.53 & \textbf{43.20} & 204.61 & 56.03 & 155.55 & 29.80 & 214.48 \\
         & + SWD ($\lambda=30$)   & \textbf{40.85} & 84.93 & 42.60 & 109.39 & \textbf{64.97} & 132.09 & \textbf{38.60} & 181.55 \\
        \bottomrule
        \end{tabular}%
    }
\end{table*}

% ================= DREAM BLOCK 32 TABLE =================
\begin{table*}[h]
    \centering
    \caption{\textbf{Dream-7B Block-32 Results.} Performance comparison with a local block constraint ($K=32$). Bold indicates the best result in each group.}
    \label{tab:app_dream_block32}
    \resizebox{\textwidth}{!}{%
        \begin{tabular}{llcccccccc}
        \toprule
        \multirow{2}{*}{\textbf{Metric}} & \multirow{2}{*}{\textbf{Strategy}} & \multicolumn{2}{c}{\textbf{HumanEval}} & \multicolumn{2}{c}{\textbf{MBPP}} & \multicolumn{2}{c}{\textbf{GSM8K}} & \multicolumn{2}{c}{\textbf{MATH500}} \\
        \cmidrule(lr){3-4} \cmidrule(lr){5-6} \cmidrule(lr){7-8} \cmidrule(lr){9-10}
         &  & \textbf{Acc} & \textbf{NFE} & \textbf{Acc} & \textbf{NFE} & \textbf{Acc} & \textbf{NFE} & \textbf{Acc} & \textbf{NFE} \\
        \midrule
        \multirow{5}{*}{Confidence} 
         & EB (baseline)    & \textbf{50.00} & 55.86 & 46.20 & 53.16 & \textbf{76.12} & 138.58 & \textbf{38.00} & 170.98 \\
         & + SWD ($\lambda=0.5$)   & 48.78 & 53.59 & \textbf{47.00} & 51.98 & 72.18 & 106.28 & 35.40 & 167.58 \\
         & + SWD ($\lambda=1.0$)   & 44.51 & 54.10 & 45.00 & 49.36 & 75.89 & 132.33 & 35.40 & 164.39 \\
         & + SWD ($\lambda=1.5$)   & 42.68 & 52.49 & 44.60 & 48.34 & 75.89 & 130.41 & 33.80 & 166.15 \\
         & + SWD ($\lambda=5.0$)   & 36.59 & 51.27 & 41.60 & 42.00 & 73.62 & 131.55 & 30.40 & 164.28 \\
        \midrule
        \multirow{3}{*}{Margin} 
         & EB ($\lambda=0$)    & \textbf{49.39} & 57.49 & 42.20 & 50.30 & \textbf{75.21} & 135.67 & 37.00 & 169.82 \\
         & + SWD ($\lambda=0.5$)   & 48.78 & 56.29 & \textbf{45.00} & 51.59 & 72.71 & 105.02 & \textbf{37.80} & 166.15 \\
         & + SWD ($\lambda=1.0$)   & 48.17 & 55.30 & 43.80 & 49.94 & 71.27 & 103.91 & 36.00 & 162.65 \\
        \midrule
        \multirow{3}{*}{Neg\_Entropy} 
         & EB ($\lambda=0$)    & 50.61 & 57.22 & 46.40 & 53.36 & 73.62 & 112.11 & 40.00 & 176.26 \\
         & + SWD ($\lambda=0.5$)   & 50.00 & 56.75 & \textbf{47.20} & \textbf{52.87} & 73.39 & 110.80 & \textbf{41.20} & 172.89 \\
         & + SWD ($\lambda=1.0$)   & \textbf{52.44} & 55.98 & \textbf{47.20} & 54.75 & \textbf{74.15} & 109.66 & 40.80 & 170.47 \\
        \bottomrule
        \end{tabular}%
    }
\end{table*}

\subsection{Code Impletation}
Key code are illustrated in Figure \ref{fig:code_comparison}. SWD only need slight change to the original code.

\begin{figure*}[t]
    \centering
    % 左栏：Baseline
    % 稍微调整宽度比例，左边代码短，可以窄一点
    \begin{minipage}[t]{0.40\textwidth} 
        \textbf{(a) Standard Confidence Top-k Decoding}
        % 【修改点】在这里添加 basicstyle=\scriptsize\ttfamily
        % 如果觉得太小，可以改成 \footnotesize
        \begin{lstlisting}[style=ebstyle, basicstyle=\scriptsize\ttfamily]
def step(logits_t, k):
    # 1. Compute Probabilities
    probs = F.softmax(logits_t, dim=-1)
    
    # 2. Base Score: Confidence
    scores, _ = probs.max(dim=-1)

    # 3. Select Top-k
    val, idx = torch.topk(scores, k)
    
    
    
    
    return idx
        \end{lstlisting}
    \end{minipage}
    \hfill
    % 右栏：SWD
    % 右边代码长，给宽一点。注意两个宽度之和不要超过 0.95\textwidth，留点余地
    \begin{minipage}[t]{0.55\textwidth}
        \textbf{(b) SWD-Enhanced Decoding (Ours)}
        % 【修改点】同样添加 basicstyle=\scriptsize\ttfamily
        \begin{lstlisting}[style=ebstyle, basicstyle=\scriptsize\ttfamily]
def swd_step(logits_t, logits_prev, k, lam):
    # 1. Compute Probabilities
    probs = F.softmax(logits_t, dim=-1)
    
    # 2. Base Score: Confidence
    scores, _ = probs.max(dim=-1)

    # 3. Apply Stability Modulator (New)
    p_prev = F.softmax(logits_prev, dim=-1)
    kl = (p_prev * (p_prev.log() - probs.log())).sum(-1)
    scores = scores * torch.exp(-lam * kl)
    
    # 4. Select Top-k
    val, idx = torch.topk(scores, k)
    return idx
        \end{lstlisting}
    \end{minipage}
    \vspace{-5pt}
    \caption{\textbf{Implementation Comparison.} Python code for a single decoding step. \textbf{Left:} Standard approach relying solely on confidence. \textbf{Right:} Our SWD method. By adding the stability modulator (Lines 9-11), we suppress unstable tokens with minimal code changes.}
    \label{fig:code_comparison}
\end{figure*}

\subsection{Analysis about CreditDecoding + SWD}

Regarding the marginal gains when combining SWD with CreditDecoding, we attribute this to information redundancy, as CreditDecoding already built a history-aware score method. Specifically, CreditDecoding maintains a Trace Credit $C_{t}^{i,v}$ for each masked position by accumulating historical logits through a process of global decay and focused enhancement: $C_{t}^{i,v} = \beta C_{t+1}^{i,v} + (p_{\theta}^{i}(v|x_{t}))^{\gamma}$ if $v = v^{\ast}$, and $\beta C_{t+1}^{i,v}$ otherwise, where $v^{\ast} = \arg \max_{v \in \mathcal{V}} p_{\theta}^{i}(v|x_{t})$. This accumulated credit is then fused with the model's raw logits, 
$\tilde{f}_\theta(x_t)_v^{i}=f_\theta(x_t)_v^{i}+\alpha\cdot\log(C_{t}^{i,v}+1)$,
which effectively applies a multiplicative prior over the likelihood to stabilize predictions against transient fluctuations. In our framework, SWD is designed to be most effective when paired with scores that directly reflect instantaneous confidence, serving as a dedicated metric for temporal stability to achieve a principled confidence-stability trade-off. However, because CreditDecoding’s base score already integrates this historical consistency to enhance and calibrate current predictions, applying SWD’s KL-based stability penalty results in a over use of the same temporal signal, leading to the observed redundancy in our discussion.

\end{document}